\documentclass{article}

\usepackage{iclr2027_conference,times}
\usepackage{amsmath,amsfonts,bm}

\def\eqref#1{equation~\ref{#1}}
\def\Eqref#1{Equation~\ref{#1}}
\def\1{\bm{1}}

\DeclareMathAlphabet{\mathsfit}{\encodingdefault}{\sfdefault}{m}{sl}
\SetMathAlphabet{\mathsfit}{bold}{\encodingdefault}{\sfdefault}{bx}{n}

\usepackage{hyperref}
\usepackage{url}
\usepackage{graphicx}
\usepackage{algorithm}
\usepackage{algpseudocode}

\usepackage{booktabs}
\usepackage{multirow}

\usepackage{amsthm}
\usepackage{amssymb}

\usepackage{tikz}
\usepackage{amsmath}
\usepackage{xcolor}
\usetikzlibrary{arrows.meta,positioning,calc,shapes.symbols}

\definecolor{radred}{HTML}{C94841}
\definecolor{radpurple}{HTML}{7057B8}
\definecolor{radorange}{HTML}{D98B12}
\definecolor{radblue}{HTML}{2F6FAE}
\definecolor{radgreen}{HTML}{2A9D62}
\definecolor{radgray}{HTML}{566573}
\definecolor{panelbg}{HTML}{F4F6F8}
\definecolor{panelbd}{HTML}{C7D7E5}
\definecolor{titleblue}{HTML}{173B66}
\definecolor{reachbd}{HTML}{E39A8C}
\definecolor{reachfill}{HTML}{FCEEEB}
\definecolor{reachtitle}{HTML}{B22A1E}
\definecolor{multbd}{HTML}{C8B5E9}
\definecolor{multfill}{HTML}{F1ECFA}
\definecolor{multtitle}{HTML}{4A2E84}
\definecolor{rewardbd}{HTML}{E5B04A}
\definecolor{rewardfill}{HTML}{FDF3DF}
\definecolor{rewardtitle}{HTML}{8F2F10}
\definecolor{actorbd}{HTML}{88C2F0}
\definecolor{actorfill}{HTML}{EAF4FD}
\definecolor{actortitle}{HTML}{154A97}
\definecolor{greenbd}{HTML}{8FD2A0}
\definecolor{greenfill}{HTML}{E6F6E7}
\definecolor{arrowgray}{HTML}{384C5A}
\definecolor{feedbackpurple}{HTML}{7E39BC}
\definecolor{loopblue}{HTML}{0E58A8}
\definecolor{orangecallout}{HTML}{F4D08D}
\definecolor{bluecallout}{HTML}{D8ECFA}
\definecolor{lightpink}{HTML}{F6D7D1}

\newtheorem{proposition}{Proposition}
\newtheorem{corollary}{Corollary}
\newtheorem{lemma}{Lemma}
\newtheorem{theorem}{Theorem}

\title{Reachability-Aware Diffusion Policy Optimization}

\author{Hikmet~Simsir\thanks{Equal contribution.},
Kutay Demiray\footnotemark[1]
\& Ozgur S. Oguz\\
Department of Computer Engineering, Bilkent University\\
\texttt{\{hikmet.simsir,kutay.demiray,ozgur.oguz\}@bilkent.edu.tr}}

\iclrfinalcopy

\begin{document}

\maketitle

% ============================================================
% ABSTRACT
% ============================================================

\begin{abstract}

Diffusion policies provide expressive action distributions for continuous-control reinforcement learning. However, safety-aware online diffusion policy optimization remains underexplored, particularly methods that use predictive reachability information without an explicit dynamics model.
We propose Reachability-Aware Diffusion Policy Optimization (RADPO), a model-free method that combines predictive first-hit safety estimation with cumulative-cost budget feedback. 
RADPO learns a discounted first-hit reachability value that captures the discounted risk of a cost event, assigns larger weight to events that occur sooner, and uses this signal to shape the reward.
A separate dual-like multiplier adjusts the shaping strength according to realized episodic costs relative to a prescribed budget. 
The diffusion actor improves through weighted denoising regression on candidate actions scored by the reward critic.
Our approach requires neither a learned dynamics model, action gradients through the critics, nor differentiation through the reverse diffusion sampler.
We establish theoretical properties of the reachability value and show that accumulated reachability penalty provides a conservative surrogate for future discounted cumulative cost. 
Across ten continuous-control safety tasks, RADPO achieves competitive reward-cost trade-offs, with substantial reductions in constraint violations on several tasks relative to the compared baselines.
Our theoretical and empirical analysis supports that combining reachability with cumulative budget feedback is a viable approach to safety-aware diffusion policies.

\end{abstract}

% ============================================================
% INTRODUCTION
% ============================================================
\section{Introduction}
  \label{sec:introduction}

  %Safe reinforcement learning aims to maximize task performance while keepingexpected cumulative safety cost below a prescribed budget. A centralchallenge is anticipating unsafe outcomes: an action may incur no immediatecost yet move the agent toward future violations. Conventional cost criticsaddress this through predictions of accumulated future cost. Discountedfirst-hit reachability offers a different signal, capturing whether a hazardwill be encountered while assigning greater weight to earlier encounters.We investigate whether this signal can guide policy improvement whilecumulative-cost feedback separately regulates the safety budget.

  A common formulation of safe reinforcement learning aims to maximize task performance while keeping expected cumulative safety cost below some budget~\citep{achiam2017constrained,stooke2020responsive}. A central challenge is anticipating unsafe outcomes: an action may incur no immediate cost yet move the agent toward future violations. Cost critics address this by estimating accumulated future cost~\citep{liu2022constrained, wu2024offpolicy}, which captures violation magnitude and repeated exposure but does not explicitly isolate when the first cost encounter occurs. Reachability-based safety estimation provides a complementary
perspective by explicitly predicting future constraint
violations~\citep{yu2022reachability,ganai2023iterative} instead. Discounted first-hit reachability captures whether a cost will be encountered while assigning greater weight to earlier encounters. We investigate whether this predictive signal can guide policy improvement while cumulative-cost feedback separately regulates the safety budget.

  % Figure~\ref{fig:reachability_motivation} illustrates how identical
  % goal-reaching success can coexist with different cost exposure.
  % In this shared navigation scenario, RADPO maintains positive cost
  % clearance, whereas SAC-Lag sometimes enters cost regions.

  Diffusion policies provide expressive action distributions for continuous control~\citep{wang2023diffusion,chi2024diffusion} and can be optimized online through value-weighted denoising regression~\citep{ma2025efficient}. However, most online diffusion-policy optimization methods focus on reward maximization~\citep{yang2023dipo,psenka2024qsm,ren2024dppo,ma2025efficient}, while incorporating predictive safety information without an explicit dynamics model is relatively underexplored. We therefore propose \textbf{Reachability-Aware Diffusion Policy Optimization (RADPO)}, a model-free approach that combines discounted first-hit reachability with cumulative-cost budget feedback.

  \begin{figure}[t]
      \centering
      \includegraphics[width=\linewidth]{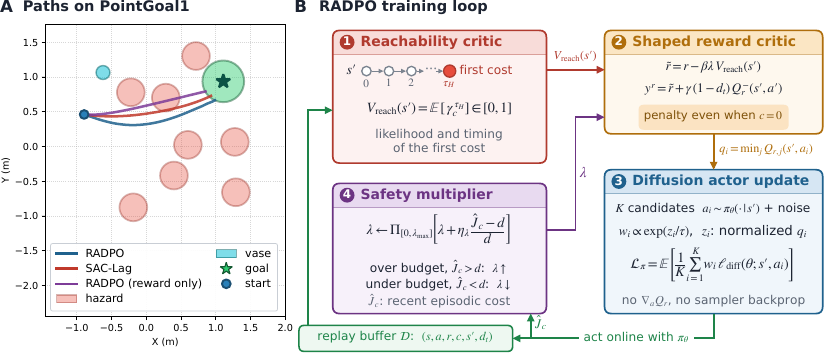}
      \vspace{-4mm}
      
  \caption{
  % \textbf{Navigation evaluation paths obtained from fully trained policies and hazard exposure on PointGoal1.}
  \textbf{(A)} Evaluation paths obtained from fully trained policies on the same fixed PointGoal1 layout. See Figure~\ref{fig:state_dist} for clearance distribution of these paths.
  \textbf{(B)} Illustrative diagram for the RADPO algorithm.
  }
      \label{fig:reachability_motivation}
    \vspace{-4mm}
  \end{figure}
  
  RADPO learns a discounted first-hit reachability value and uses the predicted reachability of successor states to shape the reward. An adaptive safety multiplier controls the shaping strength according to realized episodic costs relative to the prescribed budget, preserving information about cost magnitudes and repeated violations. The diffusion actor then improves through weighted denoising regression on candidate actions scored by the reward critic. This requires neither a learned dynamics model, action gradients through the critics, nor differentiation through the reverse diffusion sampler. This mechanism guides learning toward safer behavior. Figure~\ref{fig:reachability_motivation}.A illustrates this behavior in a navigation scenario. SAC-Lag follows a greedier path close to the hazard boundaries and occasionally enters them, whereas RADPO takes a more conservative route with greater clearance.

  Our main contributions are:
  \begin{itemize}
      %\item We introduce predictive reward shaping based on discounted first-hit reachability, coupled with cumulative-cost budget feedback, and integrate it into diffusion policy optimization through value-weighted denoising regression.

      %\item We propose a novel model-free algorithm that combines first-hit reachability-based reward shaping with cumulative-cost budget feedback to achieve safety. We integrate this algorithm into reward-only diffusion policy training to learn safe behaviors.  

      \item We introduce RADPO, a novel algorithm that combines discounted first-hit reachability with cumulative-cost budget feedback and integrates these signals into online diffusion policy optimization through value-weighted denoising regression.

      %\item We provide theoretical results for our algorithm relating accumulated reachability to discounted cumulative cost and repeated hazard exposure. We also analyze improvement in mean shaped-critic value over reweighted candidates and the safety-multiplier controller's response to observed budget violations.

      \item We provide theoretical results for the discounted reachability value, including its first-hit interpretation and bounds relating accumulated reachability to discounted cumulative cost and repeated cost exposure. We further establish finite-candidate mirror-descent guarantees and analyze the safety-multiplier controller's response to observed budget violations.
      
      \item We evaluate RADPO on ten continuous-control safety tasks,
      showing competitive reward--cost trade-offs and substantial
      reductions in constraint violations on several tasks relative
      to the compared baselines, complemented by design ablations.
  \end{itemize}
% buraya belki bi cümle
% ============================================================
% RELATED WORK
% ============================================================

\section{Related Work}

\paragraph{Constrained reinforcement learning.}
Safe reinforcement learning is commonly formulated using constrained Markov decision processes (CMDPs), where expected return is maximized subject to cumulative-cost constraints. Some methods~\citep{achiam2017constrained,yang2020projection,zhang2020first, liu2022constrained, milosevic2025embedding} enforce these constraints through constrained or trust-region policy updates, while primal-dual approaches~\citep{tessler2018reward,stooke2020responsive, wu2024offpolicy} adapt Lagrange multipliers according to observed constraint violations. Beyond changing the policy update itself, other works modify the learned safety signal used by the optimizer; for example, USPC~\citep{demiray2026uspc} replaces standard cost estimates with uncertainty-aware conservative safety surrogates. RADPO follows a similar approach, but instead introduces a distinct predictive reachability value alongside the cumulative-cost signal: episodic costs adapt the safety multiplier, while reachability shapes the reward used for policy improvement.
\vspace{-3mm}

\paragraph{Reachability-based safety and control.}
Reachability analysis attempts to determine whether target or unsafe sets can be reached under a dynamical system. Classical Hamilton-Jacobi (HJ) methods compute backward reachable or viable sets from known dynamics~\citep{mitchell2005time,bansal2017hamilton}, with subsequent work connecting reachability analysis to learning-based control and reinforcement learning~\citep{fisac2018general,akametalu2018minimum,li2025verifiable}. More recently, $\lambda$-Reachability~\citep{chen2026lambda} introduces geometric-horizon Bellman targets that interpolate between discounted local updates and long-horizon max-over-trajectory safety values. In model-free safe RL, RCRL~\citep{yu2022reachability} learns a reachability-based safety value to characterize feasible states, while RESPO~\citep{ganai2023iterative} extends this idea to stochastic systems using a learned estimate of future constraint violation. RADPO instead targets a policy-conditional discounted first-hit cost value, assigning greater weight to violations occurring sooner. This predictive signal shapes the reward, while cumulative episodic cost is handled separately through budget feedback.
\vspace{-3mm}

\paragraph{Online diffusion policy optimization.}
Several recent methods optimize diffusion policies using online reinforcement learning. DIPO~\citep{yang2023dipo} and Q-Score
Matching~\citep{psenka2024qsm} use critic gradients to improve
actions or policy scores, respectively. Other approaches optimize
a Q-weighted variational objective~\citep{ding2024qvpo},
backpropagate critic gradients through the diffusion sampling
process~\citep{wang2024dacer}, or apply PPO-style policy-gradient
optimization to the denoising process
~\citep{ren2024dppo,zou2026d2ppo,liu2026diffusion}. Most closely related to RADPO, DPMD~\citep{ma2025efficient} derives diffusion-policy improvement from policy mirror descent using value-weighted denoising regression, avoiding both critic action gradients and differentiation through the reverse sampler. RADPO builds on this policy-improvement mechanism but evaluates candidate actions using a reward critic learned from rewards shaped by predictive reachability and an adaptive safety multiplier.
\vspace{-3mm}

\paragraph{Constrained and safety-aware diffusion policies.}
Safety and constraints have also been incorporated into diffusion-based decision making. Offline methods use cost-conditioned or feasibility-guided trajectory generation~\citep{lin2024trebi,zheng2024safe}, while barrier- and constraint-guided planners incorporate safety directly into diffusion sampling or denoising~\citep{xiao2025safediffuser,mizuta2024cobl,zhang2025constrained}. In the online setting, ALGD~\citep{cheng2026algd} uses an augmented-Lagrangian energy to guide diffusion-policy optimization under safety costs. Beyond environmental safety, diffusion policies have also been subject to behavior- or trust-region constraints that limit deviation during offline or online optimization~\citep{gao2025behavior,le2026trustregion,simsir2026lpds}. RADPO instead combines fully online model-free diffusion-policy learning with a policy-conditional first-hit signal, separate episodic budget feedback, and value-weighted denoising without critic gradients or reverse-sampler differentiation.
% ============================================================
% PRELIMINARIES
% ============================================================
\section{Preliminaries}
  \label{sec:preliminaries}

  \paragraph{Constrained Markov Decision Processes}
  \label{sec:cmdp_prelim}

  We consider a CMDP with states $s\in\mathcal S$, actions
  $a\in\mathcal A$, transition kernel $P$, reward $r$, nonnegative
  cost $c$, and initial-state distribution $\rho_0$. For a policy
  $\pi$, reward discount $\gamma\in[0,1)$, and episode length $T$,
  define
  \begin{equation}
      J_r(\pi)
      =
      \mathbb E_\pi\!\left[\sum_{t=0}^{T-1}\gamma^t r_t\right],
      \qquad
      J_c(\pi)
      =
      \mathbb E_\pi\!\left[\sum_{t=0}^{T-1}c_t\right].
      \label{eq:objectives}
  \end{equation}
  The objective is to maximize $J_r(\pi)$ subject to
  $J_c(\pi)\leq d$, where $d>0$ is the episodic budget.
  A Lagrangian formulation uses
  \(
      \mathcal J(\pi,\lambda)
      =
      J_r(\pi)-\lambda\bigl(J_c(\pi)-d\bigr),
      %\qquad
      \lambda\geq0.
      \label{eq:lagrangian}
  \)
  RADPO instead uses an adaptive safety multiplier to scale predictive
  reward shaping, instead of directly optimizing $\mathcal J(\pi,\lambda)$.

    \vspace{-3mm}
  \paragraph{Diffusion Policies and Policy Mirror Descent}
  \label{sec:dpmd_prelim}

  A diffusion policy $\pi_\theta(a\mid s)$ generates actions through
  iterative denoising \citep{ho2020denoising}. We denote its
  state-conditioned denoising loss by
  $\ell_{\mathrm{diff}}(\theta;s,a)$. Let $Q(s,a)$ denote the action-value function used for policy
improvement.
  For a reference policy $\pi_{\mathrm{old}}$, KL-regularized
  \(
      \pi_{\mathrm{MD}}(a\mid s)
      =
      \frac{\pi_{\mathrm{old}}(a\mid s)\exp(Q(s,a)/\tau)}{Z(s)},
      \label{eq:mirror_descent_policy}
  \)
  where $\tau>0$ is the temperature and
  $Z(s)=\mathbb E_{a\sim\pi_{\mathrm{old}}}
  [\exp(Q(s,a)/\tau)]$.
  This update motivates value-weighted denoising regression
  \citep{ma2025efficient}:
  \begin{equation}
      \mathcal L_{\mathrm{MD}}(\theta)
      =
      \mathbb E_{\substack{s\sim\mu\\
                          a\sim\pi_{\mathrm{old}}(\cdot\mid s)}}
      \left[
          \frac{\exp(Q(s,a)/\tau)}{Z(s)}
          \ell_{\mathrm{diff}}(\theta;s,a)
      \right],
      \label{eq:dpmd_ideal}
  \end{equation}
  where $\mu$ is the state-sampling distribution. Actions and
  weights are detached during regression, avoiding differentiation
  through critics or the reverse diffusion sampler.
    \vspace{-3mm}
  \paragraph{Discounted Cost Reachability}
  \label{sec:reachability_prelim}

  Define the binary cost indicator $h_t=\1{\{c_t>0\}}$.
  For a fixed policy $\pi$ and discount $\gamma_c\in[0,1)$,
  discounted first-hit reachability satisfies
  \begin{equation}
      V_{\mathrm{reach}}^\pi(s_t)
      =
      \mathbb E_\pi\!\left[
          h_t+(1-h_t)(1-d_t)\gamma_c
          V_{\mathrm{reach}}^\pi(s_{t+1})
          \,\middle|\,s_t
      \right],
      \label{eq:reachability_recursion}
  \end{equation}
  where $d_t$ indicates termination and terminal values
  are zero. This value weights the first future cost by
  $\gamma_c$ raised to its arrival time, with zero contribution
  if no cost occurs before termination. It therefore captures
  both cost likelihood and timing, while remaining distinct
  from $J_c$, which counts cost magnitudes and repeated violations.

% ============================================================
% METHOD
% ============================================================
  \section{Reachability-Aware Diffusion Policy Optimization}
  \label{sec:method}

  RADPO combines predictive safety shaping with value-weighted
  diffusion policy optimization. A learned reachability critic
  predicts future cost exposure, while an adaptive safety multiplier
  controls the shaped penalty strength using realized episodic costs.
  The diffusion actor learns from candidates scored by the resulting
   reward critic. We describe these components below and
  describe some theoretical properties. All proofs for theoretical results are provided in Appendix~\ref{app:theory}. A simplified illustrative diagram is available in Figure~\ref{fig:reachability_motivation}.B and the full pseudocode is available in Appendix~\ref{app:pseudocode}.

  \subsection{Predictive Safety Shaping}
  \label{sec:reachability_shaping}

  Standard cost critics propagate future violations backward through Bellman updates, thus safe states leading to unsafe outcomes can already receive high predicted cost. However, cumulative cost does not explicitly distinguish how soon the next cost may be incurred. RADPO instead shapes the reward using discounted first-hit reachability, which is higher when constraint violation is more likely or expected to occur sooner. Repeatedly applying this penalty along the trajectory can thus discourage behavior that approaches violations, even when the immediate cost is zero.

  We learn an ensemble of bounded reachability critics
  $V_{reach,m}(s)\in[0,1]$, $m=1,\ldots,M$. For a replay transition
  $(s_t,a_t,r_t,c_t,s_{t+1},d_t)$, each head minimizes binary
  cross entropy against the target
  \begin{equation}
      y_{t,m}^{c}
      =
      h_t+(1-h_t)(1-d_t)\gamma_c
      V_{reach,m}^{-}(s_{t+1}),
      \label{eq:reachability_target}
  \end{equation}
  where $V_{reach,m}^{-}$ denotes its target network and
  $h_t=\1{\{c_t>0\}}$.
  The following proposition describes the exact fixed-policy
  value counterpart of  this update.

\begin{proposition}[First-hit interpretation]
  \label{prop:stopping_time}
  Fix a stationary policy $\pi$ and $0<\gamma_c<1$.
  Let
  \[
      \tau_H=\inf\{k\geq0:h_k=1\},
      \qquad \inf\varnothing:=\infty,
  \]
  and let $\tau_{\mathrm{end}}\in\mathbb N\cup\{\infty\}$
  denote the number of transitions before termination,
  with $\tau_{\mathrm{end}}=\infty$ for a nonterminating trajectory.
  Then the unique bounded fixed point of
  \Eqref{eq:reachability_recursion} satisfies
  $0\leq V_{\mathrm{reach}}^\pi(s)\leq1$ and is
  \begin{equation}
      V_{\mathrm{reach}}^\pi(s)
      =
      \mathbb E_s^\pi\!\left[
          \gamma_c^{\tau_H}
          \1{\{\tau_H<\tau_{\mathrm{end}}\}}
      \right].
      \label{eq:reachability_stopping_time}
  \end{equation}
  %where $\gamma_c^\infty:=0$.
  \end{proposition}

  Proposition~\ref{prop:stopping_time} shows that discounted first-hit reachability combines the likelihood and timing of a future constraint violation. A smaller value can reflect either a lower violation probability or a more distant violation. The proof is given in Appendix~\ref{app:stopping_time}.

  The shaping signal uses the online ensemble mean
  $V_{reach}(s)=\frac{1}{M}\sum_{m=1}^{M}V_{reach,m}(s)$:
  \begin{equation}
      \widetilde r_t
      =
      r_t-\beta\lambda V_{reach}(s_{t+1}),
      \label{eq:shaped_reward}
  \end{equation}
  where $\beta>0$ is a fixed
  shaping coefficient, and $\lambda\geq0$ is the adaptive
  safety multiplier (Sec.~\ref{sec:dual_update}). Twin reward critics are trained against
  \begin{equation}
      y_t^r
      =
      \widetilde r_t
      +(1-d_t)\gamma Q_r^{-}(s_{t+1},a_{t+1}),
      \label{eq:reward_backup}
  \end{equation}
  where $Q_r^{-}(s,a)=\min_{j\in\{1,2\}}Q_{r,j}^{-}(s,a)$
  and $a_{t+1}$ denotes the bootstrap action.
  The actor uses no separate cost term in its reweighting logits; safety is incorporated through the shaped reward.
  
  \paragraph{Relation to cumulative cost.}
To understand how reachability shaping relates to the original safety
objective, consider exact evaluation under a fixed policy $\pi$.
Define accumulated reachability and discounted cumulative cost:
  \begin{align}
      G_{\gamma,\gamma_c}^{\pi}(s,a)
      &=
      \mathbb E_{s,a}^{\pi}\!\left[
          \sum_{k\geq0}\gamma^k
          V_{\mathrm{reach}}^\pi(S_{t+k+1})
      \right],
      %\label{eq:accumulated_reachability}\\
        \:Q_c^{\pi,\delta}(s,a)
      =
      \mathbb E_{s,a}^{\pi}\!\left[
          \sum_{k\geq0}\delta^k c_{t+k}
      \right],
      \: 0<\delta<1,
      \label{eq:comparison_cost_q}
  \end{align}
  with zero continuation after termination.

  \begin{theorem}[Conservative cumulative-cost surrogate]
  \label{thm:conservative_surrogate}
  Assume $0\leq c_t\leq c_{\max}h_t$, with $c_{\max}>0$,
  and $\delta\leq\gamma<1$. Then
  \begin{equation}
      Q_c^{\pi,\delta}(s,a)
      \leq
      \mathbb E[c_t\mid s,a]
      +c_{\max}\delta G_{\gamma,\gamma_c}^{\pi}(s,a).
      \label{eq:conservative_surrogate}
  \end{equation}
  If $c_t=0$ almost surely for the initial $(s,a)$, then
  \begin{equation}
      G_{\gamma,\gamma_c}^{\pi}(s,a)
      \geq
      \frac{Q_c^{\pi,\delta}(s,a)}{c_{\max}\delta}.
      \label{eq:conservative_surrogate_safe}
  \end{equation}
  \end{theorem}

  Theorem~\ref{thm:conservative_surrogate} shows that, after separating the immediate cost, the accumulated reachability signal provides a conservative surrogate for future discounted cumulative cost: a large future cost must be accompanied by a correspondingly large accumulated reachability penalty, up to a fixed factor. When \(c_t=0\), we obtain a bound relating $G_{\gamma,\gamma_c}^{\pi}$ with $Q_c^{\pi,\delta}$. Consequently, reachability can also assign predictive penalties to safe predecessor states before the cost is incurred. The proof is in Appendix~\ref{app:conservative_surrogate}.

  Each reachability value depends only on the first subsequent cost, but RADPO applies this signal repeatedly along the trajectory. To show how this relates to repeated cost exposure, let
  \begin{equation}
      D_h^{\pi,+}(s,a)
      =
      \mathbb E_{s,a}^{\pi}\!\left[
          \sum_{k\geq1}\gamma^{k-1}h_{t+k}
      \right].
      \label{eq:repeated_cost_exposure}
  \end{equation}

  \begin{corollary}[Repeated cost exposure]
  \label{cor:repeated_costs}
  If $0\leq\gamma_c<\gamma<1$, then
  \begin{equation}
      D_h^{\pi,+}(s,a)
      \leq
      G_{\gamma,\gamma_c}^{\pi}(s,a)
      \leq
      \frac{D_h^{\pi,+}(s,a)}{1-\gamma_c/\gamma}.
      \label{eq:repeated_cost_bound}
  \end{equation}
  \end{corollary}

  Corollary~\ref{cor:repeated_costs} shows that accumulated reachability is bounded above and below by discounted repeated cost exposure, up to a fixed scaling factor. Thus, although each reachability value is first-hit, repeatedly applying it preserves sensitivity to multiple future violations. Its proof is in Appendix~\ref{app:repeated_costs}.

  \subsection{Diffusion Policy Improvement by Candidate Reweighting}
  \label{sec:dpmd_update}

  Critic regression uses replay state-action pairs, whereas
  actor targets are generated at replay next states $s'$.
  For each $s'$, we sample $K$ candidates
  $\bar a_i\sim\pi_\theta(\cdot\mid s')$ and perturb them
  before scoring, i.e.
  \(
      a_i=\Pi_{\mathcal A}(\bar a_i+\epsilon_i),
      %\qquad
      \epsilon_i\sim\mathcal N(0,\sigma_{\mathrm{exp}}^2I),
      \label{eq:candidate_perturbation}
  \)
  where $\Pi_{\mathcal A}$ clips actions to their bounds.
  The exploration scale is
  $\sigma_{\mathrm{exp}}=\sigma_0\exp(\xi)$, with fixed
  $\sigma_0$ and a learned log-scale $\xi$ adapted toward
  a target Gaussian entropy.

  Candidates receive scores
  $q_i=\min_{j\in\{1,2\}}Q_{r,j}(s',a_i)$.
  Let $\bar q=\frac{1}{K}\sum_i q_i$ and
  $\sigma=\max(\sigma_A,10^{-6})$, where $\sigma_A$
  is the advantage-normalization scale. We center,
  normalize, and clip the scores:
  \(
      z_i
      =
      \operatorname{clip}\!\left(
          \frac{q_i-\bar q}{\sigma},-3,3
      \right),
      %\qquad
      \widehat A_i=z_i/\tau,
      \label{eq:normalized_advantage}
  \)
  where $\tau>0$ controls the reweighting strength.
  The regression weights are
  \(
      w_i
      =
      K\frac{\exp(\widehat A_i)}
             {\sum_j\exp(\widehat A_j)},
  \)    %\qquad
  \(
      \frac{1}{K}\sum_i w_i=1.
      \label{eq:dpmd_weights}
  \)
  The actor minimizes the weighted denoising objective
  \begin{equation}
      \mathcal L_\pi(\theta)
      =
      \mathbb E\!\left[
          \frac{1}{K}\sum_{i=1}^{K}
          w_i\,\ell_{\mathrm{diff}}(\theta;s',a_i)
      \right],
      \label{eq:weighted_diffusion_loss}
  \end{equation}
  where the expectation covers replay successor states
  and candidate generation. Candidates and weights are
  detached, thus the update requires neither action gradients
  through critics nor differentiation through the reverse
  diffusion sampler.

  \paragraph{Finite-candidate improvement.}
  Building on the policy-mirror-descent interpretation underlying DPMD~\citep{ma2025efficient}, we describe the finite-candidate reweighting used by RADPO. Although the subsequent diffusion regression is approximate, candidate reweighting step itself admits a mirror-descent interpretation over the sampled actions.

  \begin{theorem}[Finite-candidate mirror descent]
  \label{thm:candidate_improvement}
  For a fixed candidate set, let $p_i=w_i/K$,
  $u_i=1/K$, and let $\Delta_K$ denote the probability simplex.
  Then
  \begin{equation}
      p
      =
      \arg\max_{v\in\Delta_K}
      \left\{
          \sum_i v_i z_i-\tau D_{\mathrm{KL}}(v\|u)
      \right\}.
      \label{eq:empirical_mirror_descent}
  \end{equation}
  Moreover,
  \begin{equation}
      \sum_i p_iq_i-\frac{1}{K}\sum_i q_i
      \geq
      \tau\sigma
      \left[
          D_{\mathrm{KL}}(p\|u)+D_{\mathrm{KL}}(u\|p)
      \right]
      \geq0.
      \label{eq:raw_q_improvement}
  \end{equation}
  \end{theorem}

  Theorem~\ref{thm:candidate_improvement} shows that the weights implement a KL-regularized mirror-descent update over the sampled candidates, starting from uniform weights. Moreover, despite normalization and clipping of the scores used to compute the weights, reweighted candidates cannot have a lower mean reward-critic score than the original candidate set. Safety affects this improvement through the reward critic, which is trained on the reachability-shaped reward. The proof is given in Appendix~\ref{app:candidate_improvement}.

  \subsection{Adaptive Safety Multiplier}
\label{sec:dual_update}

The reachability signal predicts first-hit cost exposure,
while the safety-multiplier controller uses the original
undiscounted episodic costs, including their magnitudes
and repeated violations. Let $\widehat J_c$ denote a recent-window estimate of
episodic cost constructed from completed training episodes,
as detailed in Appendix~\ref{app:safety_multiplier_details}.
The safety multiplier follows the projected update
\begin{equation}
    \lambda
    \leftarrow
    \Pi_{[0,\lambda_{\max}]}
    \left[
        \lambda+
        \eta_\lambda\frac{\widehat J_c-d}{d}
    \right],
    \label{eq:lambda_update}
\end{equation}
where $\eta_\lambda>0$ is the controller step size.
The multiplier increases when observed cost exceeds the
budget and decreases below budget, subject to projection.

  \paragraph{Budget-tracking property.}
The controller admits a pathwise bound on the average
episodic-cost statistic supplied to it, without requiring
convergence of the actor or critics.

  \begin{proposition}[Average observed-budget tracking]
  \label{prop:budget_tracking}
  Let $\lambda_0\in[0,\lambda_{\max}]$ and let
  $\widehat J_{c,j}$ denote the controller statistic at
  safety-multiplier update $j$. If upper clipping remains
  inactive for the first $N$ updates, then
  \begin{equation}
      \frac{1}{N}\sum_{j=0}^{N-1}\widehat J_{c,j}
      \leq
      d+\frac{d(\lambda_{\max}-\lambda_0)}
                   {\eta_\lambda N}.
      \label{eq:budget_tracking}
  \end{equation}
  If upper clipping remains inactive for all updates, then
  \begin{equation}
      \limsup_{N\to\infty}
      \frac{1}{N}\sum_{j=0}^{N-1}\widehat J_{c,j}
      \leq d.
      \label{eq:budget_tracking_asymptotic}
  \end{equation}
  \end{proposition}

  Proposition~\ref{prop:budget_tracking} shows that, as long as
the upper bound is not active, sustained
average budget violations cannot persist: the average
controller statistic exceeds $d$ by at most an
$\mathcal{O}(1/N)$ term and is asymptotically no greater than
$d$. Note that this result applies only to the statistic supplied to the
controller and does not imply per-policy or per-episode
feasibility. The proof and the corresponding bound when upper
clipping occurs are given in
Appendix~\ref{app:budget_tracking}.

% ============================================================
% EXPERIMENTS
% ============================================================
\section{Experiments}
\label{sec:experiments}

\subsection{Experimental Setup}
\label{sec:experimental_setup}

We evaluate RADPO on ten continuous-control safety tasks in Safety Gymnasium~\citep{ji2024safetygymnasiumunifiedsafereinforcement} against CAL~\citep{wu2024offpolicy}, SAC-Lag, PPO-Lag, RESPO~\citep{ganai2023iterative}, ALGD~\citep{cheng2026algd}\footnote{We compare against ALGD using the authors’ released code, which includes Lagrangian energy itself as an auxiliary loss term not described in the paper. Removing it prevents useful learning in our reproduction, therefore we report the released implementation and note this discrepancy.} and an unconstrained reward-only RADPO ablation. The experimental protocol evaluates six random seeds per configuration. The episodic cost budget is set to $d=25$, except for RESPO, which still enforces a zero-cost threshold by its formulation but is included as a hard-safety baseline. We report the mean and standard deviation of episodic return and episodic cost across seeds. The full list of hyperparameters is available in Appendix~\ref{app:hyperparameters}.

% \subsection{Main Results}
% \label{sec:main_results}

\subsection{Main Results}

Figure~\ref{fig:benchmark_results} shows that RADPO substantially reduces constraint
violations relative to reward-only RADPO while maintaining competitive returns,
especially on the navigation tasks. The reward-only variant attains high returns but generally incurs severe constraint violations. Figure~\ref{fig:pareto} summarizes the corresponding late-training
reward--cost trade-offs. RADPO achieves both higher mean return and lower mean
cost than the evaluated baselines on the two Circle tasks, and the highest mean
return on both Goal tasks. On the
Velocity tasks, RADPO is non-dominated on HalfCheetah, Humanoid,
Swimmer, and Walker2d, while baselines attain better operating points on
Ant and Hopper. Overall, RADPO provides favorable trade-offs across most
evaluated tasks.

\begin{figure*}[t]
    \centering
    \includegraphics[width=\textwidth]{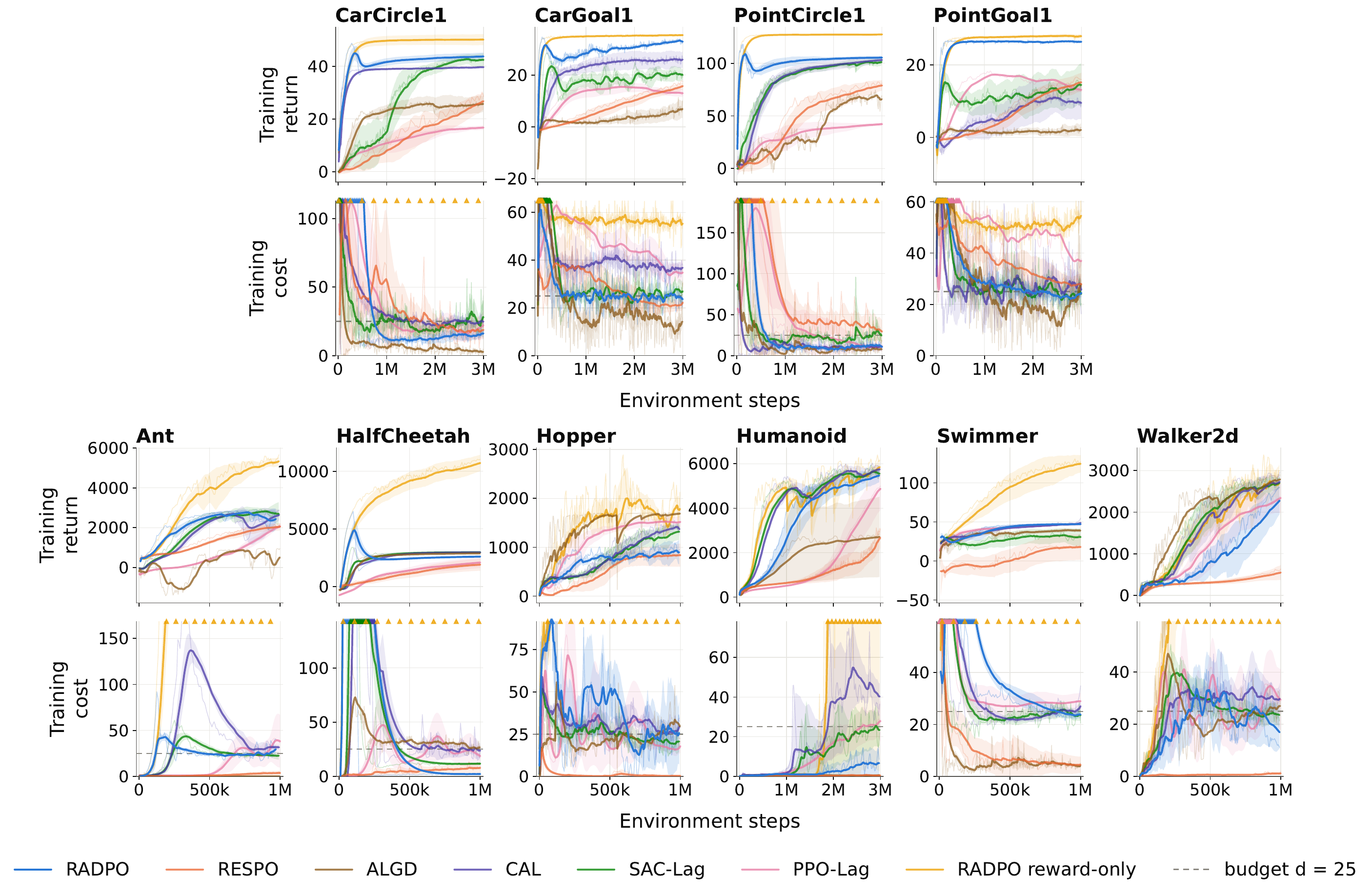}
    \vspace{-8mm}
    \caption{
    \textbf{Training performance across ten continuous-control safety tasks.}
    Episodic return (top rows) and episodic cost (bottom rows) for RADPO and baseline methods. Curves show empirical means across six random seeds; shaded regions denote one standard deviation. Dashed horizontal lines indicate the safety budget $d=25$.
    }
    \label{fig:benchmark_results}
    \vspace{-4mm}
\end{figure*}

% As shown in Figure~\ref{fig:benchmark_results}, RADPO substantially reduces safety violations compared to reward-only optimization while maintaining competitive task performance. This advantage is particularly pronounced in navigation environments, where RADPO couples high cumulative reward with low constraint costs. On velocity-constrained benchmarks, RADPO yields competitive trade-offs between performance and safety, with the exact operating balance varying by environment dynamics. In contrast, the reward-only ablation achieves high returns at the expense of severe constraint violations.

% \subsection{Reward--Cost Trade-offs}
% \label{sec:pareto}

% \begin{figure*}[t]
%     \centering
%     \includegraphics[width=\textwidth]{figures/benchmark_pareto.pdf}
%     \vspace{-8mm}
%     \caption{
%     \textbf{Late-training return versus safety cost.}
%     Markers indicate mean performance across six seeds, computed over the final ten evaluation checkpoints per run; error bars represent one standard deviation across seed averages. Optimal performance lies toward the top-left quadrant (high return, low cost). Solid gray lines trace the empirical Pareto frontier across evaluated method means, and vertical dashed lines denote the budget $d=25$.
%     }
%     \label{fig:pareto}
%     %\vspace{-4mm}
% \end{figure*}

\begin{figure*}[t]
    \centering
    \includegraphics[width=\textwidth]{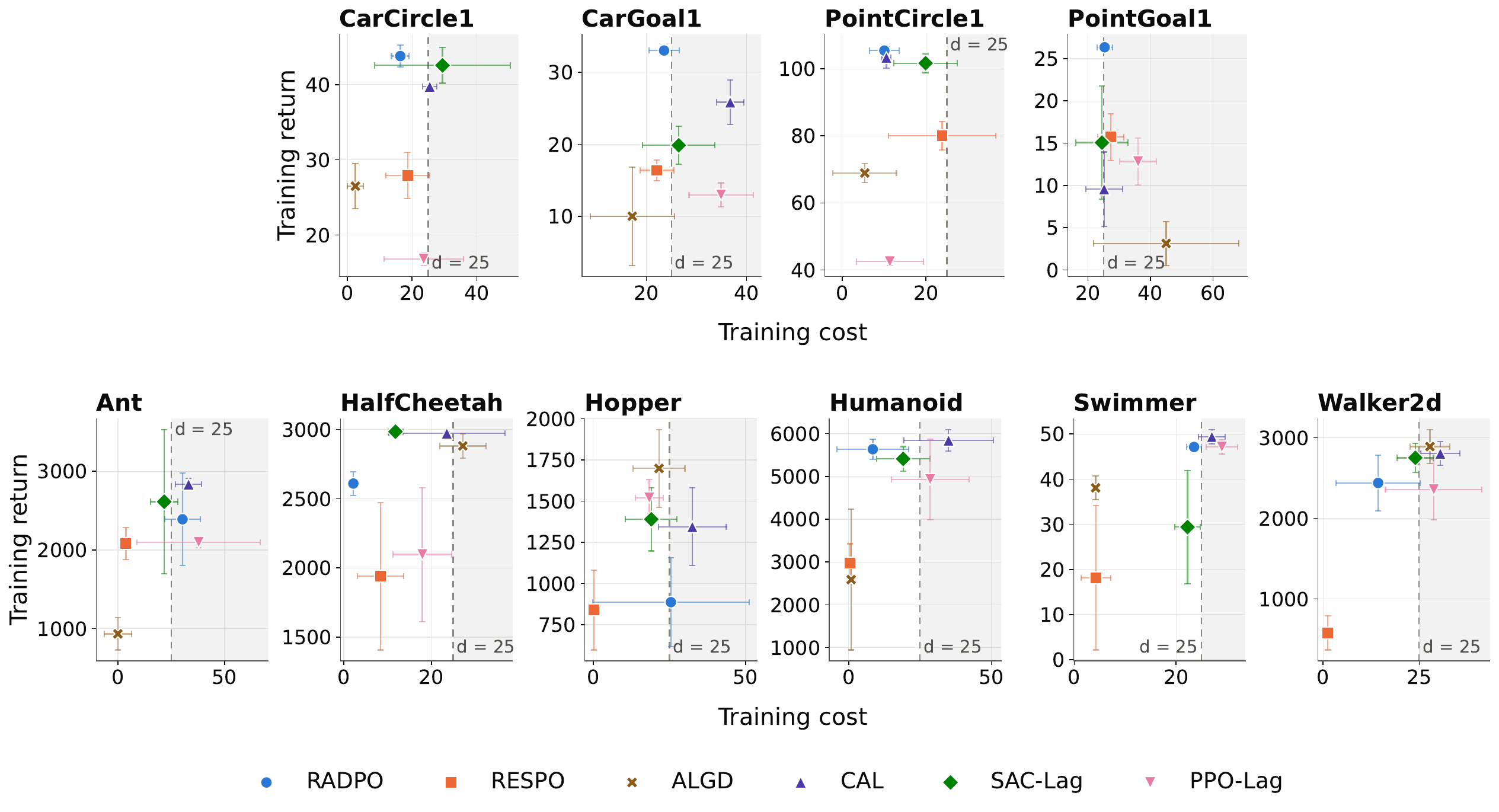}
    \vspace{-8mm}
    \caption{
    \textbf{Late-training return versus safety cost.}
    Markers indicate mean performance across seeds, computed over the final ten logged training points per run. Error bars represent one standard deviation across seed averages. Optimal performance lies toward the top-left quadrant (high return, low cost). Vertical dashed lines denote the budget $d=25$ and the shaded region is infeasible.
    }
    \label{fig:pareto}
    %\vspace{-4mm}
\end{figure*}

% Figure~\ref{fig:pareto} summarizes the trade-off between return and safety during late training, where each seed's metric is averaged across its final ten logged checkpoints. RADPO dominates baseline methods across the Circle suite, securing higher mean returns alongside lower mean costs. It similarly achieves the highest returns on CarGoal1 and PointGoal1, strictly respecting the safety threshold on CarGoal1.

% On velocity-constrained tasks, RADPO consistently establishes the empirical Pareto frontier. Across four of the six domains (HalfCheetah, Humanoid, Swimmer, and Walker2d), it occupies the optimal trade-off surface, driving returns to baseline-topping levels while strictly regulating safety costs. While baselines find more favorable operating regimes on Ant and Hopper, RADPO's Pareto-optimal profile across the majority of locomotion environments highlights its capability to balance high-speed control against rigorous velocity limits.

\subsection{Ablation Studies}
\label{sec:ablations}

We investigate two key design decisions in adapting diffusion policies to safe reinforcement learning: the choice of safety signal and the structural placement of the safety penalty. Using CarGoal1, PointCircle1, and HalfCheetah, we compare RADPO against three ablations: (1)~replacing reachability estimation with a standard discounted cost critic, (2)~applying cost-only reward shaping, and (3)~enforcing a Lagrangian safety penalty directly within candidate action weights. Additional ablation studies for hyperparameter values are available in Appendix~\ref{app:additional_ablations}.

\begin{figure*}[t]
    \centering
    \includegraphics[width=\textwidth]{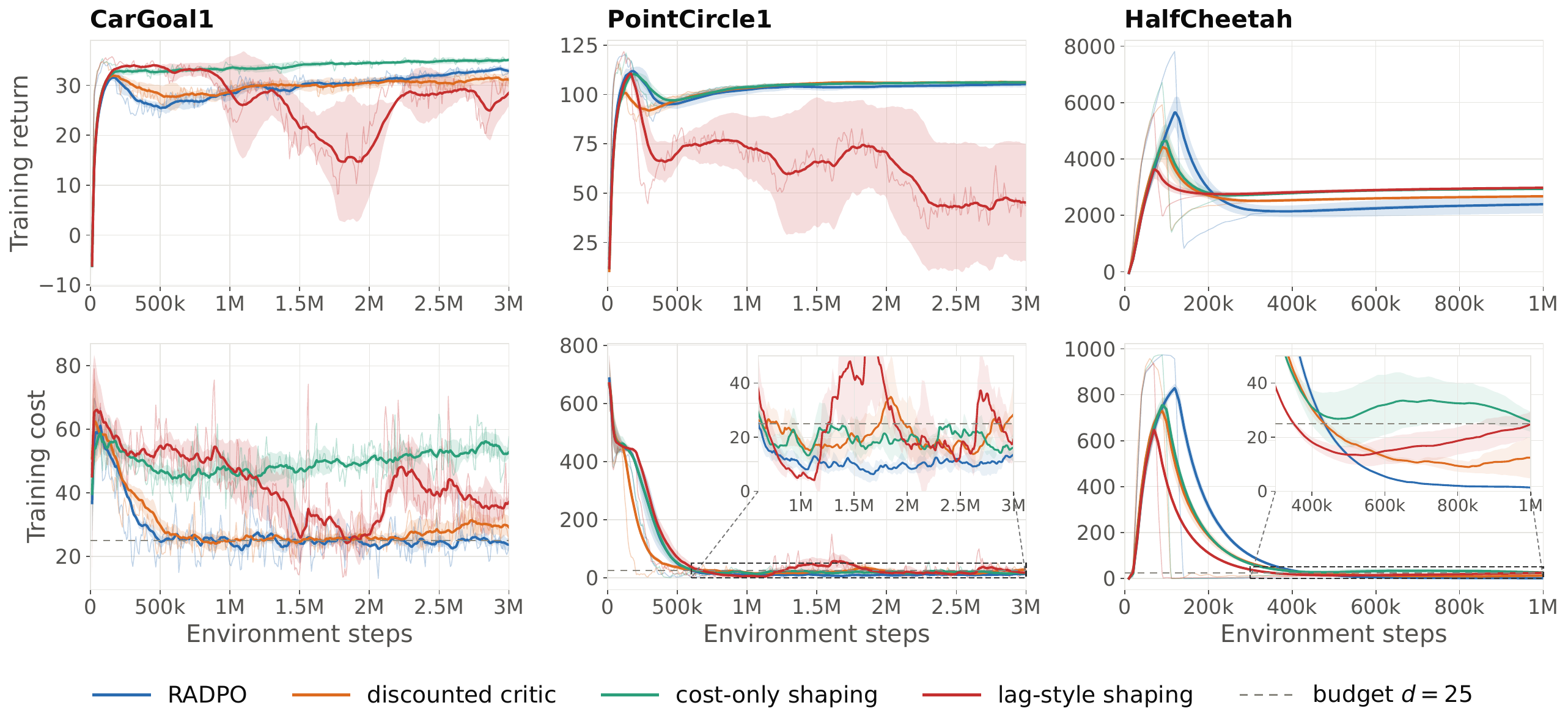}
    \vspace{-8mm}
    \caption{
    \textbf{Ablations on safety-signal formulation and penalty placement.}
    Training return (top) and episodic cost (bottom). Curves depict means across evaluated seeds with one standard deviation shaded. Solid lines show smoothing over raw trajectories (faint curves). Inset axes provide magnified views around the cost budget $d=25$.
    }
    \label{fig:ablation_results}
    \vspace{-4mm}
\end{figure*}

\paragraph{Safety-signal formulation.}
On CarGoal1, RADPO recovers target task returns while maintaining asymptotic costs close to the constraint boundary. Cost-only shaping yields higher returns but consistently breaches the safety limit, whereas discounted cost estimates result in worse returns and elevated violation rates. On PointCircle1, all three variants reach comparable asymptotic returns, but RADPO maintains lower late-training violation rates. On HalfCheetah, RADPO enforces the most conservative behavior, achieving the lowest cost among all variants at the expense of peak return. These results indicate that predictive reachability shaping offers an effective middle ground on obstacle avoidance tasks and robust constraint satisfaction in high-dimensional locomotion.
\vspace{-3mm}

\paragraph{Penalty placement in diffusion optimization.}
Penalizing candidate action weights directly via Lagrangian multipliers causes severe optimization instability: returns plummet sharply before partially recovering on CarGoal1, and drop permanently on PointCircle1. In contrast, incorporating safety into reward shaping stabilizes policy learning, yielding superior returns and tighter constraint adherence across both domains. On HalfCheetah, candidate-weight penalization preserves higher returns while operating near the constraint boundary.
% \vspace{-3mm}

%Collectively, these ablations confirm that both signal structure and algorithmic placement govern policy stability in safe diffusion RL. % Combining predictive reachability shaping with value-weighted denoising regression balances constraint enforcement and asymptotic return across diverse continuous-control benchmarks.
% ============================================================
% DISCUSSION
% ============================================================

\section{Discussion}
\label{sec:discussion}

\paragraph{Reachability interpretation.}
RADPO uses reachability as a policy-conditional prediction of
cost exposure rather than as a certificate that costs can be
avoided under an optimal controller.
Proposition~\ref{prop:stopping_time} shows that the value combines
the probability of a positive-cost event with the timing of its
first occurrence. Consequently, a smaller value can indicate either
a less likely or a later violation, and should not be interpreted
as an undiscounted violation probability. The discount $\gamma_c$
controls this temporal emphasis. RADPO therefore uses reachability as a predictive learning signal rather than as a safety certificate, since such guarantees would generally require additional assumptions in the model-free setting.
%\vspace{-3mm}

\paragraph{Separating reachability and budget feedback.}
A first-hit value does not encode violation severity or the
number of subsequent violations, hence it cannot directly represent
an episodic cost budget. RADPO instead uses the original episodic
costs to adapt $\lambda$, while reachability determines the
predictive penalty used to shape the reward. Repeated shaping
nevertheless remains sensitive to repeated exposure:
\Eqref{eq:reachability_decomposition} decomposes accumulated
reachability into discounted cost exposure and an additional
nonnegative predictive term assigned to safe predecessor states.
Theorem~\ref{thm:conservative_surrogate} and
Corollary~\ref{cor:repeated_costs} further relate this accumulated
signal to cumulative cost and repeated violations under exact
policy evaluation. Budget feedback adjusts the strength of this
surrogate but does not recover the cost-magnitude information
discarded by binarization. Because $\lambda$ is updated from
episodic cost but scales a reachability-based surrogate, it acts
as a feedback multiplier rather than as a standard dual variable
for the actor's objective.

\paragraph{Constraint structure and predictive reachability.}
We expect reachability shaping to be most useful when violations can be anticipated and avoided before cost is incurred, as in localized obstacle avoidance. In such settings, first-hit risk can distinguish trajectories that approach costs from those that maintain greater clearance, even when both have low immediate cost. The behavior in Figure~\ref{fig:reachability_motivation} is consistent with this interpretation. Its benefit may diminish when unsafe behavior persists for many steps, thus cumulative cost already provides a strong signal, or when safety directly conflicts with reward. The stronger CarGoal1 benefit and the smaller HalfCheetah effect in Figure~\ref{fig:ablation_results} are consistent with this hypothesis.
%\vspace{-3mm}

\paragraph{Limitations.} Our theoretical analysis establishes properties of individual RADPO components rather than an end-to-end feasibility result. Such a result would require stronger assumptions, for example
on function approximation, data coverage, and optimization
convergence, than we make here for RADPO. In particular, the reachability analysis assumes exact fixed-policy evaluation, whereas the implemented state-only reachability critic is deliberately trained from replay for off-policy data reuse, so its population target reflects the replay behavior mixture rather than exactly the current policy. Likewise, the candidate-improvement guarantee applies before approximate diffusion regression, and the multiplier bound applies only to the observed cost statistic. Reachability remains a surrogate rather than an exact representation of the CMDP cost. Successor-state shaping does not directly penalize the current transition cost, and first-hit shaping changes the temporal weighting of violations. For nonbinary costs, $h_t=\1\{c_t>0\}$ also discards magnitude information, although per-step costs are binary in all ten tasks considered here. Finally, budget regulation is indirect: episodic cost adjusts $\lambda$, which scales reachability shaping and influences the policy only through subsequent critic and actor updates. Transient violations and sensitivity to the controller timescale or multiplier cap can therefore still occur, as in primal-dual methods.

\section{Conclusion}
\label{sec:conclusion}

We introduced Reachability-Aware Diffusion Policy Optimization
(RADPO), a model-free approach to the relatively unexplored setting
of safe online diffusion-policy optimization under cumulative-cost
constraints. RADPO combines predictive first-hit reachability with
episodic budget feedback, while retaining candidate
reweighting without dynamics models, critic action gradients, or
reverse-process differentiation. Our theoretical analysis relates
accumulated reachability to cumulative cost and repeated violations,
and our experiments show substantial reductions in constraint
violations on several tasks while maintaining competitive returns.

\subsubsection*{AI Use Statement}

We used Generative AI tools to assist with parts of the conceptual and methodological development, mathematical claim formulation and refinement, and proof development and writing, as well as with experimental-design feedback and software implementation and debugging. In addition, they were used for literature search and summarization, manuscript drafting, and editing for typos, clarity and structure. AI tools were not used to propose or refine research hypotheses, interpret experimental results, generate, clean, or reformat datasets, perform translation, or conduct qualitative or thematic data analysis. All AI-assisted content, including mathematical arguments, code, and citations, was reviewed and verified by the authors. We take responsibility for the final content of our work.

\subsubsection*{Reproducibility Statement}
Complete proofs of the theoretical results are provided in Appendix~\ref{app:theory}. Task definitions are given in Appendix~\ref{app:task_details}, algorithmic and implementation details in Appendix~\ref{app:implementation}, and the complete experimental protocols, hyperparameters, baseline adaptations, and ablation configurations in Appendix~\ref{app:hyperparameters}. We additionally provide the code used for the experiments as supplementary material.

% ============================================================
% ACKNOWLEDGMENTS
% ============================================================
%\subsubsection*{Acknowledgments}

% TODO: Add acknowledgments and funding information.

% ============================================================
% REFERENCES
% ============================================================

\clearpage
\bibliography{iclr2027_conference}
\bibliographystyle{iclr2027_conference}

% ============================================================
% APPENDIX
% ============================================================

\appendix

\section{Theoretical Analysis}
\label{app:theory}

\subsection{Discounted reachability as a stopping-time value}
\label{app:stopping_time}

For a fixed policy $\pi$, define the Bellman operator
\begin{equation}
    (\mathcal T^\pi U)(s)
    :=
    \mathbb E_\pi
    \left[
        h_0
        +(1-h_0)(1-d_0)\gamma_c U(S_1)
        \,\middle|\,
        S_0=s
    \right].
    \label{eq:appendix_reachability_operator}
\end{equation}

\begin{lemma}[Contraction]
\label{lem:reachability_contraction}
For $0\leq\gamma_c<1$,
\[
    \|\mathcal T^\pi U-\mathcal T^\pi W\|_\infty
    \leq
    \gamma_c
    \|U-W\|_\infty.
\]
\end{lemma}

\begin{proof}
For every state $s$,
\begin{align*}
&\left|
(\mathcal T^\pi U)(s)
-
(\mathcal T^\pi W)(s)
\right|
\\
&=
\gamma_c
\left|
\mathbb E_\pi
\left[
(1-h_0)(1-d_0)
\bigl(U(S_1)-W(S_1)\bigr)
\mid S_0=s
\right]
\right|
\\
&\leq
\gamma_c
\mathbb E_\pi
\left[
|U(S_1)-W(S_1)|
\mid S_0=s
\right]
\\
&\leq
\gamma_c\|U-W\|_\infty.
\end{align*}
Taking the supremum over $s$ proves the claim. Since $\gamma_c<1$, the
operator is a contraction and therefore admits a unique bounded fixed point.
\end{proof}

\begin{proof}[Proof of Proposition~\ref{prop:stopping_time}]
By Lemma~\ref{lem:reachability_contraction},
$\mathcal T^\pi$ has a unique bounded fixed point.
Moreover, $\mathcal T^\pi$ maps functions with values in $[0,1]$
to functions with values in $[0,1]$. Iterating from the zero
function therefore shows that
\[
    0\leq V_{\mathrm{reach}}^\pi(s)\leq1.
\]

Unrolling \Eqref{eq:reachability_recursion} for $n$ transitions gives
\begin{align*}
V_{\mathrm{reach}}^\pi(s)
=
\mathbb E_s^\pi\Bigg[
&\sum_{k=0}^{n-1}
\gamma_c^k h_k
\prod_{j=0}^{k-1}(1-h_j)(1-d_j)
\\
&\quad+
\gamma_c^n
V_{\mathrm{reach}}^\pi(S_n)
\prod_{j=0}^{n-1}(1-h_j)(1-d_j)
\Bigg].
\end{align*}
The absolute value of the final term is bounded by
\[
    \gamma_c^n
    \|V_{\mathrm{reach}}^\pi\|_\infty,
\]
which converges to zero as $n\to\infty$ because
$0<\gamma_c<1$. Hence
\[
    V_{\mathrm{reach}}^\pi(s)
    =
    \mathbb E_s^\pi
    \left[
        \sum_{k\geq0}
        \gamma_c^k h_k
        \prod_{j=0}^{k-1}(1-h_j)(1-d_j)
    \right].
\]

On any realized trajectory, if the first cost occurs at
transition $\tau_H=k<\tau_{\mathrm{end}}$, then all terms vanish
except the one at $k$, which equals $\gamma_c^k$. If no cost
occurs before termination, or if neither a cost nor
termination ever occurs, every first-hit term is zero. Therefore
\[
    V_{\mathrm{reach}}^\pi(s)
    =
    \mathbb E_s^\pi
    \left[
        \gamma_c^{\tau_H}
        \1{\{\tau_H<\tau_{\mathrm{end}}\}}
    \right].
\]

For the random-horizon interpretation, let $L_{\gamma_c}$ be
independent of the trajectory with
\[
    \Pr(L_{\gamma_c}=\ell)
    =
    (1-\gamma_c)\gamma_c^\ell,
    \qquad \ell=0,1,2,\ldots.
\]
Then
\[
    \Pr(L_{\gamma_c}\geq k)=\gamma_c^k.
\]
Conditioning on $\tau_H=k$ therefore gives
\begin{align*}
&\Pr_s^\pi
\left(
\tau_H\leq L_{\gamma_c},
\ \tau_H<\tau_{\mathrm{end}}
\right)
\\
&=
\sum_{k\geq0}
\Pr_s^\pi
\left(
\tau_H=k,\,
\tau_H<\tau_{\mathrm{end}}
\right)
\Pr(L_{\gamma_c}\geq k)
\\
&=
\sum_{k\geq0}
\gamma_c^k
\Pr_s^\pi
\left(
\tau_H=k,\,
\tau_H<\tau_{\mathrm{end}}
\right)
\\
&=
V_{\mathrm{reach}}^\pi(s).
\end{align*}
\end{proof}

\begin{lemma}[Population consistency under fixed-policy sampling]
\label{lem:bce_consistency}
Fix a stationary policy $\pi$ and a target function
$\bar V:\mathcal S\rightarrow[0,1]$. Suppose that, conditional on
$S_t=s$, the transition $(A_t,h_t,d_t,S_{t+1})$ is generated
under policy $\pi$ and the environment dynamics. Define
\[
    Y
    :=
    h_t
    +(1-h_t)(1-d_t)\gamma_c\bar V(S_{t+1}).
\]
For a fixed state $s$, the minimizer of the conditional population
binary-cross-entropy loss
\[
    \mathbb E_\pi
    \left[
        -Y\log v-(1-Y)\log(1-v)
        \mid S_t=s
    \right]
\]
is
\[
    v^*(s)
    =
    \mathbb E_\pi[Y\mid S_t=s]
    =
    (\mathcal T^\pi\bar V)(s).
\]
\end{lemma}

\begin{proof}
Let
\[
    m(s):=\mathbb E_\pi[Y\mid S_t=s].
\]
Linearity of expectation gives
\[
    \mathcal L_s(v)
    =
    -m\log v-(1-m)\log(1-v).
\]
Subtracting the value at $v=m$,
\[
    \mathcal L_s(v)-\mathcal L_s(m)
    =
    D_{\mathrm{KL}}
    \bigl(
        \mathrm{Ber}(m)
        \|
        \mathrm{Ber}(v)
    \bigr)
    \geq0.
\]
Hence the population minimizer is $v=m$. Under the fixed-policy sampling assumption of the lemma, $m(s)=(\mathcal T^\pi\bar V)(s)$ by \Eqref{eq:appendix_reachability_operator}.
\end{proof}

\paragraph{Off-policy replay.}
Lemma~\ref{lem:bce_consistency} explains population
regression under fixed-policy sampling. RADPO instead trains
the state-only reachability critic from replay, consistent with
its off-policy backbone. Let $\beta_{\mathcal D}(a\mid s)$
denote the effective conditional behavior distribution represented
in the replay buffer. For fixed $\bar V$, the same BCE argument
gives
\[
    v_{\mathcal D}^*(s)
    =
    \mathbb E_{\mathcal D}[Y\mid S_t=s]
    =
    (\mathcal T^{\beta_{\mathcal D}}\bar V)(s),
\]
which generally differs from
$(\mathcal T^\pi\bar V)(s)$ for the current policy.
This is a deliberate trade-off: replay permits reuse of safety
transitions and retains previously observed cost encounters
that may become sparse as the policy improves. Since RADPO uses
$V_{reach}$ as a predictive shaping signal rather than as a safety
certificate, the implementation favors this data reuse over exact
on-policy evaluation. The resulting discrepancy should therefore
be viewed as replay-distribution approximation error.
\subsection{Conservative relation to cumulative cost}
\label{app:conservative_surrogate}

\begin{proof}[Proof of Theorem~\ref{thm:conservative_surrogate}]
Start from the cumulative cost value:
\begin{align*}
Q_c^{\pi,\delta}(s,a)
&=
\mathbb E[c_t\mid s,a]
+
\sum_{k\geq1}
\delta^k
\mathbb E[c_{t+k}\mid s,a].
\end{align*}
Because
\[
    0\leq c_{t+k}\leq c_{\max}h_{t+k},
\]
we obtain
\begin{equation}
    Q_c^{\pi,\delta}(s,a)
    \leq
    \mathbb E[c_t\mid s,a]
    +
    c_{\max}
    \sum_{k\geq1}
    \delta^k
    \mathbb E[h_{t+k}\mid s,a].
    \label{eq:appendix_cost_cost_bound}
\end{equation}

The reachability Bellman equation implies
\[
    V_{\mathrm{reach}}^\pi(s)
    \geq
    \mathbb E[h_t\mid S_t=s],
\]
because its continuation term is nonnegative. Applying the tower property,
\[
    \mathbb E
    \left[
        V_{\mathrm{reach}}^\pi(S_{t+k})
        \mid s,a
    \right]
    \geq
    \mathbb E[h_{t+k}\mid s,a].
\]
Furthermore, $\delta\leq\gamma$ implies, for $k\geq1$,
\[
    \delta^k
    =
    \delta\delta^{k-1}
    \leq
    \delta\gamma^{k-1}.
\]
Hence
\begin{align*}
\sum_{k\geq1}
\delta^k
\mathbb E[h_{t+k}\mid s,a]
&\leq
\delta
\sum_{k\geq1}
\gamma^{k-1}
\mathbb E
\left[
V_{\mathrm{reach}}^\pi(S_{t+k})
\mid s,a
\right]
\\
&=
\delta
G_{\gamma,\gamma_c}^{\pi}(s,a).
\end{align*}
Substituting into
\Eqref{eq:appendix_cost_cost_bound} proves
\Eqref{eq:conservative_surrogate}.

If $c_t=0$ almost surely under $(s,a)$, the immediate term vanishes and
\Eqref{eq:conservative_surrogate_safe} follows by rearranging.
\end{proof}

The source of the additional conservatism can be made explicit. Applying the
reachability Bellman equation yields
\begin{align}
\label{eq:reachability_decomposition}
G_{\gamma,\gamma_c}^{\pi}(s,a)
&=
D_h^{\pi,+}(s,a)
+
\gamma_c
P_{\mathrm{pred}}^\pi(s,a),
%\label{eq:predictive_uplift}
\\
P_{\mathrm{pred}}^\pi(s,a)
&:=
\mathbb E_{s,a}^{\pi}
\left[
\sum_{k\geq0}
\gamma^k
(1-h_{t+k+1})
(1-d_{t+k+1})
V_{\mathrm{reach}}^\pi(S_{t+k+2})
\right]
\geq0.
\nonumber
\end{align}
Thus the difference between accumulated reachability and immediate discounted
cost exposure consists of a nonnegative predictive term assigned to safe
predecessor states.

\subsection{Repeated cost exposure}
\label{app:repeated_costs}

\begin{proof}[Proof of Corollary~\ref{cor:repeated_costs}]
From the reachability Bellman recursion,
\[
    V_{\mathrm{reach}}^\pi(S_{t+k})
    \geq
    \mathbb E
    \left[
        h_{t+k}
        \mid S_{t+k}
    \right].
\]
Taking conditional expectation from the initial $(s,a)$, multiplying by the
corresponding powers of $\gamma$, and summing gives
\[
    G_{\gamma,\gamma_c}^{\pi}(s,a)
    \geq
    D_h^{\pi,+}(s,a).
\]

For the upper bound, first-hit reachability is pathwise bounded by the
discounted sum of all future costs:
\[
    V_{\mathrm{reach}}^\pi(S_n)
    \leq
    \mathbb E
    \left[
        \sum_{j\geq0}
        \gamma_c^j h_{n+j}
        \,\middle|\,
        S_n
    \right].
\]
Therefore
\begin{align*}
G_{\gamma,\gamma_c}^{\pi}(s,a)
&\leq
\sum_{k\geq0}
\sum_{j\geq0}
\gamma^k\gamma_c^j
\mathbb E_{s,a}^{\pi}
[h_{t+k+1+j}].
\end{align*}
Set $u=k+1+j$. For each fixed $u\geq1$, the coefficient multiplying
$\mathbb E[h_{t+u}]$ is
\[
    \sum_{j=0}^{u-1}
    \gamma^{u-1-j}\gamma_c^j
    =
    \gamma^{u-1}
    \sum_{j=0}^{u-1}
    \left(
        \frac{\gamma_c}{\gamma}
    \right)^j.
\]
Since $\gamma_c<\gamma$,
\[
    \sum_{j=0}^{u-1}
    \left(
        \frac{\gamma_c}{\gamma}
    \right)^j
    \leq
    \frac{1}{1-\gamma_c/\gamma}.
\]
Hence
\[
    G_{\gamma,\gamma_c}^{\pi}(s,a)
    \leq
    \frac{1}{1-\gamma_c/\gamma}
    D_h^{\pi,+}(s,a),
\]
which proves the result.
\end{proof}

\subsection{Finite-candidate mirror descent}
\label{app:candidate_improvement}

\begin{lemma}[Empirical Gibbs variational identity]
\label{lem:gibbs_variational}
Let
\[
    p_i
    =
    \frac{\exp(z_i/\tau)}
    {\sum_j\exp(z_j/\tau)},
    \qquad
    u_i=\frac1K.
\]
Then
\[
    p
    =
    \arg\max_{v\in\Delta_K}
    \left\{
        \langle v,z\rangle
        -
        \tau D_{\mathrm{KL}}(v\|u)
    \right\},
\]
and
\[
    \langle p-u,z\rangle
    =
    \tau
    \left[
        D_{\mathrm{KL}}(p\|u)
        +
        D_{\mathrm{KL}}(u\|p)
    \right].
\]
\end{lemma}

\begin{proof}
Let
\[
    Z_u
    :=
    \frac1K
    \sum_i
    \exp(z_i/\tau).
\]
Then
\[
    p_i
    =
    \frac{
        u_i\exp(z_i/\tau)
    }{
        Z_u
    },
\]
so
\[
    \log\frac{p_i}{u_i}
    =
    \frac{z_i}{\tau}
    -
    \log Z_u.
\]
For arbitrary $v\in\Delta_K$,
\begin{align*}
D_{\mathrm{KL}}(v\|p)
&=
D_{\mathrm{KL}}(v\|u)
-
\frac1\tau
\langle v,z\rangle
+
\log Z_u.
\end{align*}
Rearranging,
\[
    \langle v,z\rangle
    -
    \tau D_{\mathrm{KL}}(v\|u)
    =
    \tau\log Z_u
    -
    \tau D_{\mathrm{KL}}(v\|p),
\]
which is uniquely maximized at $v=p$.

Moreover,
\begin{align*}
&D_{\mathrm{KL}}(p\|u)
+
D_{\mathrm{KL}}(u\|p)
\\
&=
\sum_i
(p_i-u_i)
\log\frac{p_i}{u_i}
\\
&=
\frac1\tau
\sum_i
(p_i-u_i)z_i,
\end{align*}
because $\sum_i(p_i-u_i)=0$.
\end{proof}

\begin{proof}[Proof of Theorem~\ref{thm:candidate_improvement}]
The first statement follows directly from
Lemma~\ref{lem:gibbs_variational}.

For the raw-$Q$ bound, define
\[
    \sigma:=\max(\sigma_A,10^{-6}),
    \qquad
    R_i:=q_i-\sigma z_i.
\]
Then
\[
    q_i=\sigma z_i+R_i.
\]
For fixed $\bar q$, the scalar function
\[
    R(q)
    =
    q
    -
    \sigma
    \operatorname{clip}
    \left(
        \frac{q-\bar q}{\sigma},
        -3,3
    \right)
\]
is non-decreasing in $q$. The softmax probability $p_i$ is also non-decreasing
with $q_i$, because both clipping and exponentiation preserve order.
Consequently,
\[
    (p_i-p_j)(R_i-R_j)\geq0
    \qquad
    \forall i,j.
\]
Using the pairwise covariance identity,
\[
    \sum_i
    (p_i-u_i)R_i
    =
    \frac1K
    \sum_{i<j}
    (p_i-p_j)(R_i-R_j)
    \geq0.
\]
Therefore
\begin{align*}
\sum_i(p_i-u_i)q_i
&=
\sigma
\sum_i(p_i-u_i)z_i
+
\sum_i(p_i-u_i)R_i
\\
&\geq
\sigma
\sum_i(p_i-u_i)z_i.
\end{align*}
Applying Lemma~\ref{lem:gibbs_variational},
\[
    \sum_i(p_i-u_i)q_i
    \geq
    \tau\sigma
    \left[
        D_{\mathrm{KL}}(p\|u)
        +
        D_{\mathrm{KL}}(u\|p)
    \right]
    \geq0.
\]
This is \Eqref{eq:raw_q_improvement}.
\end{proof}

\subsection{Safety-multiplier budget tracking}
\label{app:budget_tracking}

For update index $j$, define the normalized observed budget error
\[
    e_j
    :=
    \frac{
        \widehat J_{c,j}-d
    }{
        d
    }.
\]
Define the amount discarded by upper clipping as
\[
    u_j
    :=
    \left[
        \lambda_j+\eta_\lambda e_j-\lambda_{\max}
    \right]_+.
\]

\begin{proposition}[Capped budget-tracking bound]
\label{prop:capped_budget_tracking}
For every $N\geq1$,
\begin{equation}
    \frac1N
    \sum_{j=0}^{N-1}
    \widehat J_{c,j}
    \leq
    d
    +
    \frac{d}{\eta_\lambda N}
    \left[
        \lambda_{\max}-\lambda_0
        +
        \sum_{j=0}^{N-1}u_j
    \right].
    \label{eq:capped_budget_tracking}
\end{equation}
\end{proposition}

\begin{proof}
Let
\[
    x_j
    :=
    \lambda_j+\eta_\lambda e_j
\]
be the unprojected update and define the lower-projection correction
\[
    \ell_j:=[-x_j]_+.
\]
Projection onto $[0,\lambda_{\max}]$ can be written exactly as
\[
    \lambda_{j+1}
    =
    \lambda_j
    +
    \eta_\lambda e_j
    +
    \ell_j
    -
    u_j.
\]
Rearranging and summing,
\begin{align*}
\eta_\lambda
\sum_{j=0}^{N-1}e_j
&=
\lambda_N-\lambda_0
-
\sum_{j=0}^{N-1}\ell_j
+
\sum_{j=0}^{N-1}u_j
\\
&\leq
\lambda_{\max}-\lambda_0
+
\sum_{j=0}^{N-1}u_j,
\end{align*}
because $\lambda_N\leq\lambda_{\max}$ and $\ell_j\geq0$.
Substituting
\[
    e_j
    =
    \frac{\widehat J_{c,j}-d}{d}
\]
and rearranging proves
\Eqref{eq:capped_budget_tracking}.
If the upper projection never activates, then $u_j=0$ for all $j$, reducing
\Eqref{eq:capped_budget_tracking} exactly to
\Eqref{eq:budget_tracking}.
\end{proof}

At any converged interior fixed point, $0 < \lambda^* < \lambda_{\max}$,
with limiting controller statistic $\widehat J_c^*$, the projection is
inactive. Therefore
\[
    \lambda^*
    =
    \lambda^*
    +
    \eta_\lambda
    \frac{\widehat J_c^* - d}{d}
\]
implies $\widehat J_c^* = d$. Boundary fixed points instead satisfy the
corresponding one-sided conditions:
\[
    \lambda^* = 0 \Rightarrow \widehat J_c^* \leq d,
    \qquad
    \lambda^* = \lambda_{\max}
    \Rightarrow \widehat J_c^* \geq d.
\]
These conditions resemble complementary slackness, but $\lambda$ is not a
standard Lagrange multiplier for the same objective optimized by the actor,
consequently we use them only as an interpretation of the controller equilibrium.

\section{Task Details}
\label{app:task_details}

We evaluate on ten environments from Safety Gymnasium~\citep{ji2024safetygymnasiumunifiedsafereinforcement}: four safe-navigation
tasks obtained by combining the Point and Car agents with the Goal1 and
Circle1 tasks, and six velocity-constrained MuJoCo locomotion tasks.

\paragraph{Navigation agents.}
The \textbf{Point} agent is a simple planar robot with two control inputs:
forward/backward motion and rotation about the vertical axis. The
\textbf{Car} agent uses two independently actuated parallel wheels together
with a free-rolling rear wheel, such that steering and forward motion must be
coordinated through the two wheel controls. We evaluate both agents on the
Goal1 and Circle1 tasks.

\paragraph{Navigation tasks.}
In \textbf{Goal1}, the agent is rewarded for reducing its distance to a
target and receives an additional reward when the target is reached, after
which a new target is generated. The Level-1 environment contains eight
hazard regions and one movable vase; entering a hazard produces safety cost,
whereas the vase is not part of the safety constraint. In
\textbf{Circle1}, the agent is rewarded for moving around a prescribed
circular path while remaining close to its radius. Two virtual boundary walls
restrict the admissible region, and crossing from the interior to the exterior
produces safety cost.

% Optional representative task figure:
% \begin{figure}[t]
%     \centering
%     \includegraphics[width=.32\linewidth]{figures/point_goal1.png}
%     \includegraphics[width=.32\linewidth]{figures/car_circle1.png}
%     \includegraphics[width=.32\linewidth]{figures/halfcheetah_velocity.png}
%     \caption{Representative Safety Gymnasium environments.
%     Left: Goal1, in which the agent navigates to targets while avoiding
%     hazards. Center: Circle1, in which the agent follows a circular path
%     while remaining within the safety boundaries. Right: a Safe Velocity
%     task, in which locomotion performance is optimized subject to an
%     agent-specific velocity limit.}
%     \label{fig:task_examples}
% \end{figure}

\paragraph{Velocity-constrained locomotion.}
The remaining six environments augment the standard Gymnasium MuJoCo
locomotion tasks with a velocity constraint. We consider \textbf{Ant}, a
3D quadruped; \textbf{HalfCheetah}, a planar running robot;
\textbf{Hopper}, a planar monoped; \textbf{Humanoid}, a 3D biped;
\textbf{Swimmer}, an articulated swimming agent; and \textbf{Walker2d},
a planar biped. The underlying task rewards locomotion according to the
corresponding MuJoCo environment, while Safety Gymnasium produces a binary
cost whenever the agent's velocity exceeds an environment-specific threshold.
For Swimmer the constraint is applied to the $x$-axis velocity; for the other
agents the relevant planar velocity magnitude is used.

\section{Additional Ablations}
\label{app:additional_ablations}
% hyperparameter ablation will come here

To justify our hyperparameter choices, we show our ablations for reachability discount $\gamma_c$, shaping coefficient $\beta$ and candidates per state $K$ in Figure~\ref{fig:hyperparameter_ablation}. As expected, higher values of $\gamma_c$ yield a more conservative policy, especially for HalfCheetah. Similarly, higher $\beta$ results in a more conservative policy; for PointGoal1, when $\beta$ is too high the reward signal is drowned out and the policy fails to learn. Finally for $K$, we only possibly see a slight improvement in safety for higher values in exchange for a higher computational cost. 

\begin{figure*}[t]
    \centering
    \includegraphics[width=\textwidth]{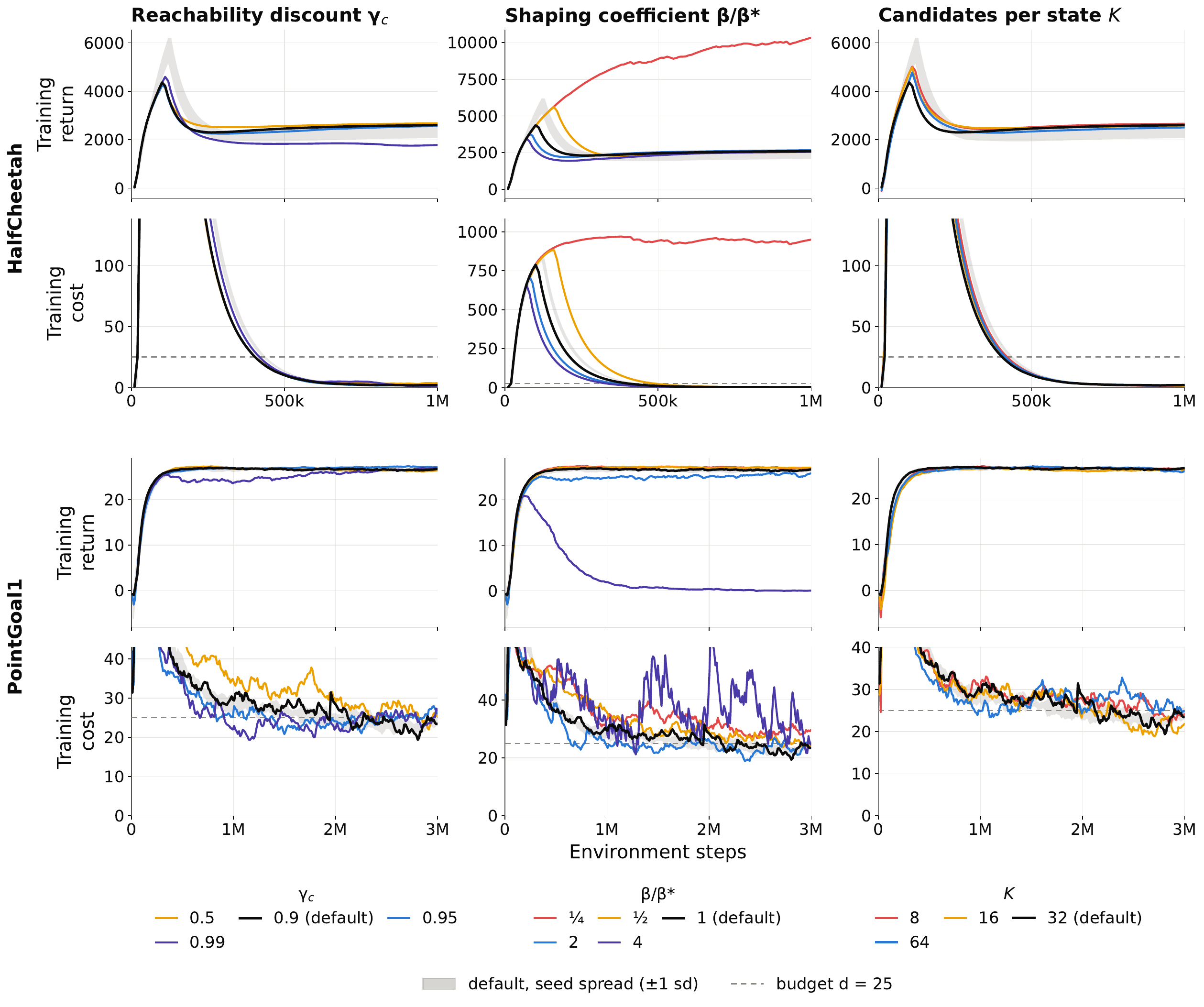}
    \vspace{-8mm}
    \caption{
    \textbf{Ablations on hyperparameter values.}
    Ablation results for reachability discount $\gamma_c$, shaping coefficient $\beta$ and candidates per state $K$. Curves depict means across evaluated seeds with one standard deviation shaded.
    }
    \label{fig:hyperparameter_ablation}
    \vspace{-4mm}
\end{figure*}
\section{Computational Costs}
\label{app:computational_costs}
% RADPO, ALGD, DPMD-Lag pointgoal1 computational costs for phases of the algorithms (in terms of hours) will come here

Table~\ref{tab:pointgoal1-compute} shows computational cost comparison of ALGD, Lagrangian-style shaping and RADPO on PointGoal1. Runs were taken on a computer with one RTX 3090 Ti and i9-12900F. Lagrangian-style and RADPO runs were taken with our codebase. ALGD uses the official implementation provided by their authors, modified only to use Safety Gymnasium and log the timing details shown in this table. Note that these runs retain the 10/0.25/0.25 actual learner updates per transition for ALGD, Lagrangian-style and RADPO respectively, as described in their default configurations, which means that ALGD has $40\times$ more updates per environment step, hence the significantly longer times. 

\begin{table}[t]
\centering
\small
% \begin{tabular}{@{}lrrrrr@{}}
% \toprule
% & \multicolumn{5}{c}{Time per environment transition (ms)} \\
% \cmidrule(l){2-6}
% Method & \shortstack{Sample\\acquisition}
%        & \shortstack{Critic\\optimization}
%        & \shortstack{Diffusion-policy\\optimization}
%        & \shortstack{Other}
%        & \shortstack{Overall} \\
% \midrule
% ALGD & 3.298 & 16.057 & 78.892 & 41.524 & 138.420 \\
% Lagrangian-style  & 0.418 & 0.545 & 0.459 & 3.387 & 2.919 \\
% RADPO & 0.616 & 0.423 & 0.406 & 2.697 & 2.962 \\
% \bottomrule
% \end{tabular}
\begin{tabular}{@{}lrrrrr@{}}
\toprule
& \multicolumn{5}{c}{Time per environment transition (ms)} \\
\cmidrule(l){2-6}
Method
& \shortstack{Sample\\acquisition}
& \shortstack{Critic\\optimization}
& \shortstack{Diffusion-policy\\optimization}
& \shortstack{Other\\learning}
& \shortstack{Overall} \\
\midrule
ALGD
& 3.298 & 16.057 & 78.892 & 41.524 & 139.771 \\
Lagrangian-style
& 0.418 & 0.545 & 0.459 & 3.387 & 4.809 \\
RADPO
& 0.616 & 0.423 & 0.406 & 2.697 & 4.142 \\
\bottomrule
\end{tabular}
\caption{Computational cost on PointGoal1. All entries are ms per environment step.}
\label{tab:pointgoal1-compute}
\end{table}

\section{Additional Figures}
\label{app:additional_figures}
Figure~\ref{fig:state_dist} shows the state distribution of the paths shown in Figure~\ref{fig:reachability_motivation}.

\begin{figure}[t]
      \centering
      \includegraphics[width=0.5\linewidth]{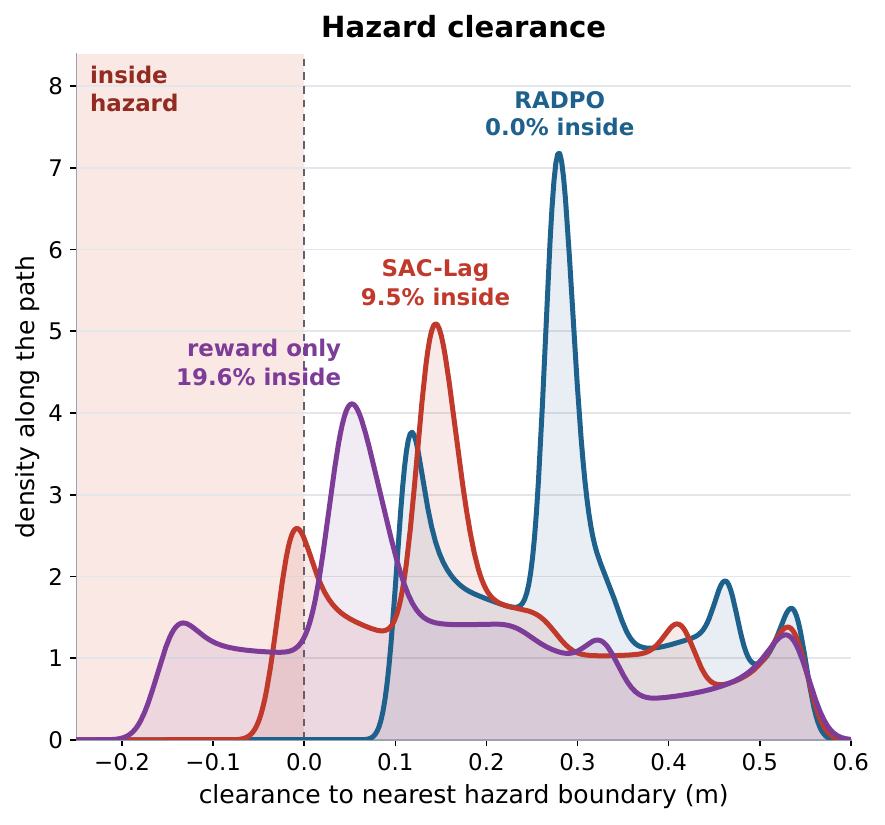}
      \vspace{-4mm}
      
  \caption{
    Distribution of states for trajectories in Figure~\ref{fig:reachability_motivation}.A.
  }
      \label{fig:state_dist}
    \vspace{-4mm}
  \end{figure}

\section{Implementation Details}
\label{app:implementation}

% TODO:
% Full hyperparameter table
% Network architectures
% Diffusion schedule
% Optimizer settings
% Replay buffer
% Training/evaluation details

\subsection{Safety-multiplier Cost Estimate}
\label{app:safety_multiplier_details}

The safety-multiplier controller operates on completed training
episodes. For each collection event $j$ with a nonempty set
of completed episodes $\mathcal E_j$, we compute the mean
undiscounted episodic cost
\begin{equation}
    \bar C_j
    =
    \frac{1}{|\mathcal E_j|}
    \sum_{e\in\mathcal E_j}
    \sum_{t=0}^{L_e-1}c_t^{(e)}.
    \label{eq:completion_batch_cost}
\end{equation}
The controller statistic averages these completion-batch
means over a recent window $\mathcal W$:
\begin{equation}
    \widehat J_c
    =
    \frac{1}{|\mathcal W|}
    \sum_{j\in\mathcal W}\bar C_j.
    \label{eq:controller_cost_estimate}
\end{equation}

\subsection{Pseudocode}

We present the pseudocode for RADPO training in Algorithm~\ref{alg:radpo}.

\label{app:pseudocode}
\begin{algorithm}[t]
  \caption{Reachability-Aware Diffusion Policy Optimization (RADPO)}
  \label{alg:radpo}
  \begin{algorithmic}[1]
  \Require Cost budget $d$, candidate count $K$
  \State Initialize actor, critics, target networks, replay buffer
  $\mathcal D$, and safety multiplier $\lambda$
  \While{training}
      \State Collect experience using candidate selection and exploration
      \State Store transitions in $\mathcal D$
      \State Adapt $\lambda$ using recent episodic costs
      \hfill \Eqref{eq:lambda_update}
      \For{each gradient update}
          \State Sample $(s,a,r,c,s',d_t)$ from $\mathcal D$
          \State Compute reachability targets
          \hfill \Eqref{eq:reachability_target}
          \State Shape rewards:
          $\widetilde r = r-\beta\lambda V_{reach}(s')$
          \State Compute reward-critic targets
          \hfill \Eqref{eq:reward_backup}
          \State Sample and perturb $K$ actor candidates at $s'$
          \State Score candidates with $Q_r$ and compute weights $w_i$
          \State Update reachability and reward critics
            \State Update actor using
          $\mathcal L_\pi=\mathbb E_{s'}[
          \frac{1}{K}\sum_{i=1}^{K}
          w_i\,\ell_{\mathrm{diff}}(\theta;s',a_i)]$
          with detached candidates and weights
          \State Update exploration scale, statistics, and targets
          as scheduled
      \EndFor
  \EndWhile
  \end{algorithmic}
\end{algorithm}

% =====================================================================================
%  RADPO -- Appendix: Experimental details and hyperparameters
%
%  WHERE THIS GOES
%    Paste as a new \section AFTER \section{Implementation Details} in the appendix
%    stream, i.e. it becomes Appendix D (A Theoretical Analysis, B Task Details,
%    C Implementation Details, D this).  It defines \label{app:hyperparameters}, which
%    is what Section 5.1 already cites -- that resolves the current "Appendix ??".
%
%  ONE DELETION IN THE MAIN FILE
%    The TODO comment block inside \section{Implementation Details}
%      % TODO: Full hyperparameter table / Network architectures / Diffusion schedule
%      %       Optimizer settings / Replay buffer / Training-evaluation details
%    is fully covered here and can be removed.
%
%  PACKAGES
%    booktabs and multirow are NOT yet in the preamble -- add
%      \usepackage{booktabs}
%      \usepackage{multirow}
%    next to \usepackage{graphicx}.  Nothing else is needed: every table fits the
%    text width unscaled.  \Eqref comes from math_commands.tex.
%
%  All \ref and \citep keys below match the current main file; nothing else to fix.
% =====================================================================================

\section{Extension to Graded Costs and Multiple Constraints}
\label{app:extensions}

Since the RADPO algorithm as described assumes a single constraint with a binary per-step cost signal, in this section we discuss extension of RADPO to graded (continuous) costs and multiple constraint settings. Note that our Safety Gymnasium tasks have single constraints with binary costs by default. 

\subsection{Graded Costs}
The bounded-cost assumption in Theorem~\ref{thm:conservative_surrogate} does not require binary costs.
For any nonnegative graded cost $c_t \in [0,c_{\max}]$ and
$h_t=\mathbf{1}\{c_t>0\}$, the condition
$c_t \leq c_{\max} h_t$ holds automatically.
Thus, the cumulative-cost bound continues to apply without modification.
Binarization only removes cost-magnitude information from the predictive
reachability signal; the original graded episodic costs are still retained
by the safety-multiplier update. Hence RADPO can be applied directly to
bounded graded costs, although its reachability shaping distinguishes the
timing and likelihood of positive-cost events rather than their severity.

\subsection{Multiple Constraints}
For constraints $J_{c_i}(\pi)\leq d_i$, $i=1,\ldots,m$, RADPO can be
extended by maintaining a reachability critic and adaptive safety
multiplier for each constraint. Each multiplier is updated
from its corresponding realized episodic cost and budget, while reward
shaping becomes
\begin{equation}
    \widetilde r_t
    =
    r_t
    -
    \sum_{i=1}^{m}
    \beta_i \lambda_i V_{reach}^i(s_{t+1}).
\end{equation}
The reachability and cumulative-cost relations of
Theorem~1 then apply separately to each constraint. This preserves the
individual budgets and predictive safety signals rather than collapsing
multiple constraints into a single scalarized cost.

\section{Experimental Details and Hyperparameters}
\label{app:hyperparameters}

This appendix specifies every setting used to produce Figures~\ref{fig:benchmark_results}--\ref{fig:ablation_results}.
Section~\ref{app:hp:protocol} fixes the tasks, budgets and evaluation protocol shared by all methods;
Section~\ref{app:hp:radpo} gives RADPO's hyperparameters, split into settings that are constant across
all ten tasks and the small set that varies per task; Section~\ref{app:hp:baselines} gives the baseline
configurations; and Section~\ref{app:hp:ablation} gives the ablation arms of
Section~\ref{sec:ablations}.

\subsection{Interaction budgets and evaluation protocol}
\label{app:hp:protocol}

All experiments use Safety Gymnasium~\citep{ji2024safetygymnasiumunifiedsafereinforcement}: six MuJoCo
velocity-constrained locomotion tasks and four navigation tasks, described in
Appendix~\ref{app:task_details}. Every task uses an episode horizon of
$T=1000$ steps and the episodic cost budget $d=25$, except where a baseline's own formulation
prescribes otherwise (Section~\ref{app:hp:baselines}). Per-step costs are binary on all ten tasks, so
the binarisation $h_t=\mathbf{1}{\{c_t>0\}}$ discards no magnitude information here.

\paragraph{Evaluation.} Training curves report the episodic return and cost of the on-policy
training rollouts; the latter are used to construct the
safety-multiplier statistic. Deterministic evaluation runs $20$ episodes at the
cadence in Table~\ref{tab:tasks}. The late-training scatter of Figure~\ref{fig:pareto} averages each
run's final ten logged training points, then reports the mean and standard deviation over the runs of an
arm.

\begin{table}[h]
\centering
\caption{Tasks, interaction budgets and evaluation cadence.}
\label{tab:tasks}
\small
\begin{tabular}{llrrr}
\toprule
Family & Task & Horizon $T$ & Env.\ steps & Eval.\ every \\
\midrule
\multirow{6}{*}{Velocity}
 & SafetyAntVelocity-v1          & 1000 & $1\mathrm{M}$ & $25\mathrm{k}$ \\
 & SafetyHalfCheetahVelocity-v1  & 1000 & $1\mathrm{M}$ & $25\mathrm{k}$ \\
 & SafetyHopperVelocity-v1       & 1000 & $1\mathrm{M}$ & $25\mathrm{k}$ \\
 & SafetyHumanoidVelocity-v1     & 1000 & $3\mathrm{M}$ & $25\mathrm{k}$ \\
 & SafetySwimmerVelocity-v1      & 1000 & $1\mathrm{M}$ & $25\mathrm{k}$ \\
 & SafetyWalker2dVelocity-v1     & 1000 & $1\mathrm{M}$ & $25\mathrm{k}$ \\
\midrule
\multirow{4}{*}{Navigation}
 & SafetyCarCircle1-v0           & 1000 & $3\mathrm{M}$ & $100\mathrm{k}$ \\
 & SafetyCarGoal1-v0             & 1000 & $3\mathrm{M}$ & $100\mathrm{k}$ \\
 & SafetyPointCircle1-v0         & 1000 & $3\mathrm{M}$ & $100\mathrm{k}$ \\
 & SafetyPointGoal1-v0           & 1000 & $3\mathrm{M}$ & $100\mathrm{k}$ \\
\bottomrule
\end{tabular}
\end{table}

\paragraph{Implementation and compute.} All methods except RESPO and ALGD are implemented in a single JAX code base
that shares the environment wrappers, replay buffer, vectorised sampler, evaluation loop and logging,
hence cross-method differences come from the algorithms and not from the harness. RESPO and ALGD are run from their
authors' released implementations. Runs were executed on a heterogeneous pool of consumer and
workstation NVIDIA GPUs (GTX~1080, RTX~4090, RTX~5090, RTX~6000~Ada/Pro), one run per GPU.

\subsection{RADPO}
\label{app:hp:radpo}

\subsubsection{Settings shared by all ten tasks}

Table~\ref{tab:radpo-shared} lists every RADPO hyperparameter that is identical on all ten tasks;
Algorithm~\ref{alg:radpo} in Appendix~\ref{app:pseudocode} gives the update order these settings
parameterise.
The diffusion actor, the twin reward critics and the reachability critics are all $3\times256$ MLPs
with Mish activations; the actor additionally takes the denoising-step embedding as input
(Section~\ref{sec:dpmd_prelim}). Updates are Adam. The actor's learning rate is linearly annealed from
$3\times10^{-4}$ to $3\times10^{-5}$ between gradient steps $2.5\times10^{4}$ and $7.5\times10^{4}$;
the critic learning rates are constant. The exploration scale is
$\sigma_{\mathrm{exp}}=\sigma_0\exp(\xi)$, with $\xi$ adapted every $250$ gradient steps toward a
fixed Gaussian entropy target proportional to the action dimension. Safety enters RADPO only through
reachability shaping of the reward critic, $\widetilde r_t=r_t-\beta\lambda V_{reach}(s_{t+1})$, on
every task; the candidate weights of \Eqref{eq:dpmd_weights} carry no separate cost term.

\begin{table}[h]
\centering
\caption{RADPO hyperparameters held constant across all ten tasks.}
\label{tab:radpo-shared}
\small
\begin{tabular}{llc}
\toprule
Group & Hyperparameter & Value \\
\midrule
\multirow{5}{*}{Networks}
 & Critic / reachability-critic MLP        & $3\times256$, Mish \\
 & Diffusion actor MLP                     & $3\times256$, Mish \\
 & Denoising steps                         & $20$ \\
 & Noise schedule                          & cosine \\
 & Reward critics (twin)                   & $2$ \\
\midrule
\multirow{8}{*}{Optimisation}
 & Optimiser                               & Adam \\
 & Critic learning rate                    & $3\times10^{-4}$ \\
 & Actor learning rate                     & $3\times10^{-4}\!\rightarrow\!3\times10^{-5}$ (linear) \\
 & Replay capacity                         & $10^{6}$ transitions \\
 & Batch size                              & $256$ \\
 & Warm-up (uniform-random) transitions    & $30\mathrm{k}$ \\
 & Actor / target update period            & every $2$ gradient steps \\
 & Target smoothing $\tau_{\text{Polyak}}$ & $0.005$ \\
\midrule
\multirow{4}{*}{Candidate reweighting}
 & Candidates per state $K$                & $32$ \\
 & Reweighting temperature $\tau$          & $1.0$ \\
 & Score clip                              & $[-3,3]$ \\
 & Advantage scale $\sigma_A$              & EMA of batch std, coefficient $10^{-3}$ \\
\midrule
\multirow{3}{*}{Exploration}
 & Entropy-target update period            & every $250$ gradient steps \\
 & $\xi$ learning rate                     & $3\times10^{-4}$ \\
 & Action projection $\Pi_{\mathcal{A}}$   & clip to box bounds \\
\midrule
\multirow{6}{*}{Safety}
 & Shaping signal                           & $V_{reach}(s_{t+1})$ (reachability) \\
 & Reachability ensemble size $M$           & $8$ \\
 & Reachability loss                        & binary cross-entropy \\
 & Reachability discount $\gamma_c$         & $0.9$ \\
 & Reward discount $\gamma$                 & $0.99$ \\
\midrule
\multirow{6}{*}{Multiplier controller}
 & Budget $d$                              & $25$ \\
 & $\lambda$ initialisation                & $0$ \\
 & $\lambda$ lower bound                   & $0$ \\
 & $\lambda$ cap $\lambda_{\max}$          & $10$ \\
 & Cost window $|\mathcal{W}|$             & $10$ completion batches (App.~\ref{app:safety_multiplier_details}) \\
 & Update trigger                          & each collection event with $\ge 1$ completed episode \\
\midrule
\multirow{2}{*}{Evaluation}
 & Deterministic episodes                  & $20$ \\
 & Candidate selection at deployment       & highest reward-critic score \\
\bottomrule
\end{tabular}
\end{table}

\subsubsection{Per-task settings}

Table~\ref{tab:radpo-per-task} lists the settings that vary across tasks. We tune the
update-to-data ratio, initial exploration temperature $\alpha_0$, exploration scale
$\sigma_0$, shaping coefficient $\beta$, and controller step size $\eta_\lambda$.

Among these, $\eta_\lambda$ is the least transferable across tasks. The update
\[
\lambda \leftarrow
\Pi_{[0,\lambda_{\max}]}
\left[
\lambda + \eta_\lambda \frac{\hat{J}_c-d}{d}
\right]
\]
scales with the relative budget violation $(\hat{J}_c-d)/d$, which depends on the
task. For example, CarGoal1 requires $\eta_\lambda=0.055$, whereas the Circle tasks
use $0.005$. Swimmer also requires a relatively large value,
$\eta_\lambda=0.10$, because its episodes do not terminate early. The controller,
which updates only when episodes are completed, therefore receives substantially
fewer updates during training, and a smaller step size adapts too slowly within the
available interaction budget.

\begin{table}[h]
\centering
\caption{Per-task RADPO hyperparameters. UTD is the number of optimiser updates per collected
transition, $\alpha_0$ the initial exploration temperature
($\sigma_{\mathrm{exp}}=\sigma_0\alpha_0$ at initialisation), $\sigma_0$ the exploration scale,
$\beta$ the shaping coefficient and $\eta_\lambda$ the controller step size. Everything else is as in
Table~\ref{tab:radpo-shared}.}
\label{tab:radpo-per-task}
\small
\begin{tabular}{lrrrrr}
\toprule
Task & UTD & $\alpha_0$ & $\sigma_0$ & $\beta$ & $\eta_\lambda$ \\
\midrule
Ant          & $1.0$  & $0.2$ & $0.5$ & $1.5$   & $0.010$ \\
HalfCheetah  & $1.0$  & $0.2$ & $0.5$ & $1.5$   & $0.010$ \\
Hopper       & $0.6$  & $0.2$ & $0.5$ & $0.8$   & $0.002$ \\
Humanoid     & $1.0$  & $0.2$ & $0.1$ & $2.4$   & $0.120$ \\
Swimmer      & $0.2$  & $0.2$ & $0.5$ & $0.025$ & $0.100$ \\
Walker2d     & $0.2$  & $0.2$ & $0.5$ & $1.5$   & $0.010$ \\
\midrule
CarCircle1   & $1.0$  & $0.5$ & $0.1$ & $0.05$  & $0.005$ \\
CarGoal1     & $1.0$  & $0.5$ & $0.1$ & $0.05$  & $0.055$ \\
PointCircle1 & $1.0$  & $0.5$ & $0.1$ & $0.05$  & $0.005$ \\
PointGoal1   & $0.25$ & $0.5$ & $0.1$ & $0.10$  & $0.005$ \\
\bottomrule
\end{tabular}
\end{table}

\subsection{Baselines}
\label{app:hp:baselines}

All baselines use the same budget $d=25$ and the same interaction budgets as Table~\ref{tab:tasks}, except RESPO (which enforces a zero-cost threshold by formulation). Each baseline is initialized from its published configuration; any adaptations or tuning are stated below.

\subsubsection{SAC-Lagrangian}

Soft actor--critic with a Lagrangian multiplier updated by Adam on $\log\lambda$ from the realised
episodic cost. Two backbones are used, matching the two task families
(Table~\ref{tab:saclag}): the velocity family uses the standard SAC backbone; the Goal tasks use the
slower actor of the reference safe-RL implementations, which we verified is \emph{necessary} on these
tasks (raising the actor learning rate from $5\times10^{-6}$ to $3\times10^{-4}$ collapses
PointGoal1 to approximately zero return). The cost critic is floored at zero: because the actor minimises
$\lambda Q_c$ without a lower bound, an unfloored cost critic diverges (we observed $Q_c\!\approx\!-800$
against a theoretical range of $[0,100]$, with the run's return collapsing).

\begin{table}[h]
\centering
\caption{SAC-Lagrangian. ``Velocity'' covers the six velocity tasks and the two Circle tasks;
``Goal'' covers CarGoal1 and PointGoal1.}
\label{tab:saclag}
\small
\begin{tabular}{lcc}
\toprule
Hyperparameter & Velocity / Circle & Goal \\
\midrule
Critic / actor MLP              & $3\times256$, GELU & $2\times256$, ReLU \\
Actor learning rate             & $3\times10^{-4}$   & $5\times10^{-6}$ \\
Critic learning rate            & $3\times10^{-4}$   & $1\times10^{-3}$ \\
Entropy temperature $\alpha$    & learned, $\alpha_0=0.2$ & fixed $10^{-5}$ \\
$\alpha$ learning rate          & $3\times10^{-4}$   & --- \\
Policy delay                    & $1$                & $2$ \\
Gradient-norm clip              & ---                & $40$ \\
Multiplier learning rate        & $1\times10^{-4}$   & $1\times10^{-5}$ \\
Multiplier initialisation       & $10^{-3}$          & $0$ \\
Multiplier upper bound          & $100$              & $100$ \\
Multiplier update               & per completed episode & per gradient step after $100\mathrm{k}$ steps \\
Cost window                     & $1$                & $10$ \\
Non-negative cost critic        & yes                & yes \\
Batch size / replay             & $256$ / $10^{6}$   & $256$ / $10^{6}$ \\
UTD ratio                       & $0.25$             & $0.25$ \\
Warm-up transitions             & $30\mathrm{k}$     & $10\mathrm{k}$ \\
Discounts $\gamma,\gamma_c$     & $0.99$, $0.99$     & $0.99$, $0.99$ \\
Polyak $\tau$                   & $0.005$            & $0.005$ \\
\bottomrule
\end{tabular}
\end{table}

\subsubsection{CAL}
CAL~\citep{wu2024offpolicy} is initialized from its published configuration.
Because the published navigation experiments use fixed layouts, we retune the SAC
backbone for the randomized Circle tasks. Specifically, we change the actor
learning rate from $5\times10^{-6}$ to $3\times10^{-4}$, replace the fixed
entropy temperature $\alpha=10^{-5}$ with a learned temperature, increase the
network depth from two to three hidden layers, change the policy delay from $2$
to $1$, and adjust the discounted per-step cost threshold from $1.0$ to $0.5$.
CAL's safety-specific coefficients and multiplier update are otherwise unchanged.
The Goal tasks retain the published backbone.

Common settings are a four-member cost-critic ensemble with ReLU activations,
target updates every two gradient steps, multiplier initialization at zero,
discounts $\gamma=\gamma_c=0.99$ except $\gamma_c=0.9$ on the Circle tasks,
batch size $256$, replay capacity $10^6$, UTD ratio $0.25$, and Polyak
coefficient $\tau=0.005$. Table~\ref{tab:cal} gives the per-task
settings.

% On the six velocity tasks CAL is run on the standard SAC backbone ($3\times256$ GELU, actor learning
% rate $3\times10^{-4}$, learned $\alpha$, policy delay $1$). On the two Circle tasks we additionally
% retuned the SAC backbone CAL publishes (actor learning rate $5\times10^{-6}\!\rightarrow\!3\times10^{-4}$,
% fixed $\alpha=10^{-5}\!\rightarrow$ learned, $2\!\rightarrow\!3$ hidden layers, policy delay
% $2\!\rightarrow\!1$, and the discounted per-step cost threshold $1.0\!\rightarrow\!0.5$) leaving
% CAL's safety machinery ($k$, $c$, $d$, the multiplier step) untouched. This retune is strictly in
% CAL's favour: on PointCircle1 it raises CAL's independently evaluated return by $19\%$ and lowers its
% cost by $70\%$, and it makes CAL the only arm in this paper feasible on its mean there. On
% the two Goal tasks CAL keeps its published backbone ($2\times256$ ReLU, actor learning rate
% $5\times10^{-6}$, fixed $\alpha=10^{-5}$, policy delay $2$, critic learning rate $10^{-3}$, gradient
% clip $40$).

\begin{table}[h]
\centering
\caption{CAL per-task settings. $k$ is the cost-critic UCB coefficient, $c$ the rectification
coefficient. ``done-mask'' indicates that termination is masked in the cost backup;
``$Q_c\!\ge\!0$'' that the cost critic is floored at zero.}
\label{tab:cal}
\small
\begin{tabular}{lrrlrl}
\toprule
Task & $k$ & $c$ & Multiplier LR & Internal $d$ & Extra \\
\midrule
Ant          & $0.0$ & $0.01$ & $1\times10^{-5}$ & $25$ & --- \\
HalfCheetah  & $0.0$ & $0.3$  & $1\times10^{-4}$ & $25$ & $Q_c\!\ge\!0$ \\
Hopper       & $0.5$ & $0.2$  & $1\times10^{-5}$ & $25$ & done-mask, $Q_c\!\ge\!0$ \\
Humanoid     & $0.0$ & $0.01$ & $1\times10^{-5}$ & $25$ & done-mask \\
Swimmer      & $0.0$ & $0.01$ & $1\times10^{-5}$ & $25$ & --- \\
Walker2d     & $0.0$ & $0.2$  & $1\times10^{-4}$ & $25$ & done-mask \\
\midrule
CarCircle1   & $0.0$ & $0.05$ & $5\times10^{-7}$ & $25$ & retuned backbone, cost threshold $0.5$ \\
CarGoal1     & $0.0$ & $0.3$  & $5\times10^{-7}$ & $25$ & published backbone \\
PointCircle1 & $0.0$ & $10.0$ & $5\times10^{-7}$ & $25$ & retuned backbone, cost threshold $0.5$ \\
PointGoal1   & $0.0$ & $10.0$ & $5\times10^{-7}$ & $25$ & published backbone \\
\bottomrule
\end{tabular}
\end{table}

\subsubsection{PPO-Lagrangian}

The on-policy reference is run at its published settings, except that we reduce
the number of epochs for the navigation tasks so that their total interaction
budget matches the $3$M environment steps used by the other methods. No per-task
hyperparameter tuning is applied. Table~\ref{tab:ppolag} lists the resulting
configuration; rollout lengths follow the published protocol, while the
navigation interaction budget differs only through this change in the number of
epochs.

\begin{table}[h]
\centering
\caption{PPO-Lagrangian.}
\label{tab:ppolag}
\small
\begin{tabular}{lccc}
\toprule
Hyperparameter & Velocity & Humanoid & Navigation \\
\midrule
Policy / value MLP        & $2\times64$, tanh & $2\times64$, tanh & $2\times64$, tanh \\
Learning rate             & $3\times10^{-4}$ & $3\times10^{-4}$ & $3\times10^{-4}$ \\
Steps per epoch           & $2048$ & $6000$ & $20000$ \\
Epochs                    & $500$  & $500$  & $150$ \\
Total environment steps   & $1.024\mathrm{M}$ & $3\mathrm{M}$ & $3\mathrm{M}$ \\
Update iterations / epoch & $40$ & $40$ & $40$ \\
Minibatch size            & $64$ & $64$ & $64$ \\
Clip ratio                & $0.2$ & $0.2$ & $0.2$ \\
Target KL (early stop)    & $0.02$ & $0.02$ & $0.02$ \\
$\gamma$ / GAE $\lambda$  & $0.99$ / $0.95$ & $0.99$ / $0.95$ & $0.99$ / $0.95$ \\
Cost GAE $\lambda$        & $0.95$ & $0.95$ & $0.95$ \\
Entropy coefficient       & $0$ & $0$ & $0$ \\
Gradient-norm clip        & $40$ & $40$ & $40$ \\
Critic $\ell_2$ coefficient & $10^{-3}$ & $10^{-3}$ & $10^{-3}$ \\
Multiplier learning rate  & $0.035$ & $0.035$ & $0.035$ \\
Multiplier initialisation & $10^{-3}$ & $10^{-3}$ & $10^{-3}$ \\
Lagrange window           & $100$ epochs & $100$ epochs & $100$ epochs \\
\bottomrule
\end{tabular}
\end{table}

% \subsubsection{RESPO}

% RESPO~\citep{ganai2023iterative} is run from the authors' released implementation at its published
% per-task defaults; we changed no hyperparameter. RESPO enforces a \emph{zero}-cost threshold by
% formulation rather than the episodic budget $d=25$ used by every other method, which is why it
% appears far to the left in Figure~\ref{fig:pareto} at correspondingly lower return. Its runs log only
% on-policy rollout statistics, so RESPO appears in the training figures and is absent from the
% deterministic-evaluation panels.

\subsubsection{RESPO}
\label{app:respo}

We adapt the authors' released PyTorch
RESPO implementation~\citep{ganai2023iterative} to Safety Gymnasium,
preserving its policy objective, reachability targets, and multiplier
update. The Gaussian actor and three value networks use two
256-unit tanh layers. Training uses four MPI workers, full-batch
updates, and linearly decaying learning rates. 
We performed hyperparameter tuning on selected tasks, varying the actor learning rate,
rollout size, target KL, and initial policy standard deviation.
We additionally used observation normalization for Humanoid, which improves the performance compared to the default setting.
The selected settings in Table~\ref{tab:respo_hyperparameters}
were then held fixed across the subsequent training seeds.
All tasks use a 1,000-step episode limit and action repeat one;
observation normalization is applied only to Humanoid.

We retain RESPO's released training constraint of zero.
The common episodic cost threshold $d=25$ is used only for
reporting and assessing feasibility, not for its multiplier update.
Learning curves show training-rollout return and cost.

\begin{table}[h]
\centering
\caption{RESPO per-task settings. Budgets and rollout sizes count total
environment transitions across four workers. Actor LR and $\sigma_0$
are the initial learning rate and policy standard deviation.
Observation normalization is used only for Humanoid.}
\label{tab:respo_hyperparameters}
\small
\begin{tabular}{lrrlrr}
\toprule
Task & Budget & Rollout & Actor LR & $\sigma_0$ & Target KL \\
\midrule
Ant          & 1M & 10,000 & $6\times10^{-4}$ & $0.2$      & $0.03$ \\
HalfCheetah  & 1M & 10,000 & $6\times10^{-4}$ & $0.2$      & $0.03$ \\
Hopper       & 1M & 10,000 & $6\times10^{-4}$ & $0.2$      & $0.03$ \\
Humanoid     & 3M & 5,000  & $1\times10^{-4}$ & $0.2$      & $0.03$ \\
Swimmer      & 1M & 10,000 & $6\times10^{-4}$ & $0.2$      & $0.03$ \\
Walker2d     & 1M & 10,000 & $3\times10^{-4}$ & $0.2$      & $0.03$ \\
\midrule
CarCircle1   & 3M & 30,000 & $1\times10^{-4}$ & $0.2$      & $0.01$ \\
CarGoal1     & 3M & 30,000 & $1\times10^{-4}$ & $e^{-0.5}$ & $0.01$ \\
PointCircle1 & 3M & 30,000 & $3\times10^{-4}$ & $e^{-0.5}$ & $0.01$ \\
PointGoal1   & 3M & 30,000 & $3\times10^{-4}$ & $e^{-0.5}$ & $0.01$ \\
\bottomrule
\end{tabular}
\end{table}

% ~~~~~~~~~~~~~~~~~~~~~~~~
% ~~~~~~~~~~~~~~~~~~~~~~~~
% ~~~~~~~~~~~~~~~~~~~~~~~~
% ~~~~~~~~~~~~~~~~~~~~~~~~

\subsubsection{ALGD}
\label{app:hp:algd}

We run ALGD~\citep{cheng2026algd} from the authors' released PyTorch implementation
(commit \texttt{1f6820b}), with the training harness adapted so that ALGD sees the same
environments, horizons, budgets and evaluation as the other methods
(Table~\ref{tab:algd-protocol}). Table~\ref{tab:algd-shared} lists the released
hyperparameters, which we keep, and Table~\ref{tab:algd-per-task} the per-task settings.

\begin{table}[h]
\centering
\caption{ALGD: released training protocol versus the protocol used here.}
\label{tab:algd-protocol}
\small
\begin{tabular}{lll}
\toprule
 & Released code & This paper \\
\midrule
Environments          & Safety-Gym (\texttt{Safexp-*}), MuJoCo & Safety Gymnasium (all ten tasks) \\
Navigation layout     & fixed for every episode and seed & resampled every episode \\
Episode horizon       & $400$ (navigation)               & $1000$ \\
Goal reached          & episode terminates               & new goal spawned, episode continues \\
Cost budget           & $10$ (navigation) / raw velocity & $d=25$ on all tasks \\
Evaluation            & $1$ episode                      & $20$ deterministic episodes \\
\bottomrule
\end{tabular}
\end{table}

\paragraph{Cost threshold.} ALGD constrains a discounted cost critic, and its released code
converts an episodic budget into a per-step threshold
$h=d\,(1-\gamma_c^{H})/\bigl((1-\gamma_c)H\bigr)$. We apply this conversion on all ten tasks
with $d=25$, $\gamma_c=0.99$ and $H$ the number of decisions per episode, giving
$h\approx2.50$ for the Goal and velocity tasks ($H=1000$) and $h\approx9.19$ for the Circle
tasks ($H=250$ at action repeat $4$).

\begin{table}[h]
\centering
\caption{ALGD hyperparameters (released defaults, shared by the DDPM and VE-SDE agents).}
\label{tab:algd-shared}
\small
\begin{tabular}{llc}
\toprule
Group & Hyperparameter & Value \\
\midrule
\multirow{4}{*}{Networks}
 & Reward critics (twin)            & $2\times256$, ReLU \\
 & Cost-critic ensemble             & $4$ members, $2\times256$, SiLU \\
 & Score / noise network            & $3\times128$, ReLU, time embedding \\
 & Denoising steps $K$              & $5$ \\
\midrule
\multirow{2}{*}{Diffusion}
 & DDPM noise schedule              & linear $\beta\in[10^{-4},0.02]$ \\
 & VE-SDE noise scale               & geometric $\sigma\in[0.01,1.0]$ \\
\midrule
\multirow{3}{*}{Augmented Lagrangian}
 & Penalty coefficient $\rho$       & $1.0$ \\
 & Multiplier initialisation        & $\lambda_0=0.6931$ \\
 & Multiplier update                & Adam on $\log\lambda$, lr $3\times10^{-4}$ \\
\midrule
\multirow{2}{*}{Score estimation}
 & Monte-Carlo samples              & $4$ \\
 & Score temperature $\beta$        & $1.0$ \\
\midrule
\multirow{7}{*}{Optimisation}
 & Optimiser                        & Adam \\
 & Actor / critic / cost-critic lr  & $3\times10^{-4}$ \\
 & Batch size / replay              & $256$ / $10^{6}$ \\
 & Discounts $\gamma,\gamma_c$      & $0.99$, $0.99$ \\
 & Polyak $\tau$ / target period    & $0.005$ / every $2$ updates \\
 & Warm-up (uniform-random) steps   & $5\mathrm{k}$ \\
 & Updates per environment step     & $10$ (default), see Table~\ref{tab:algd-per-task} \\
\bottomrule
\end{tabular}
\end{table}

\begin{table}[h]
\centering
\caption{ALGD per-task settings. UTD is gradient updates per environment step.}
\label{tab:algd-per-task}
\small
\begin{tabular}{lll}
\toprule
Task & Agent & UTD \\
\midrule
Ant          & DDPM   & $10$   \\
HalfCheetah  & DDPM   & $10$   \\
Hopper       & DDPM   & $10$   \\
Humanoid     & DDPM   & $10$ \\
Swimmer      & DDPM   & $10$ \\
Walker2d     & DDPM   & $10$   \\
\midrule
CarCircle1   & VE-SDE & $10$   \\
CarGoal1     & DDPM   & $10$ \\
PointCircle1 & VE-SDE & $10$   \\
PointGoal1   & VE-SDE & $10$ \\
\bottomrule
\end{tabular}
\end{table}

\subsection{Ablation configurations}
\label{app:hp:ablation}

The ablations of Figure~\ref{fig:ablation_results} are one-factor changes to the corresponding constrained
arm of Table~\ref{tab:radpo-per-task}; every other hyperparameter is held fixed.

\begin{itemize}
\item \textbf{Discounted critic.} Replaces the reachability critic with a standard discounted cost
critic, keeping the same successor-state reward-shaping structure.
\item \textbf{Cost-only shaping.} Shapes with the realised cost,
$\tilde{r}_t=r_t-\beta\lambda c_t$, instead of the predicted reachability of the successor state.
\item \textbf{Lag-style shaping.} Removes reward shaping entirely and puts the cost term into the
candidate weights, i.e.\ the logits become $\hat{A}_i-\lambda\hat{A}^c_i$ with the cost advantage
clipped symmetrically to $[-3,3]$ to match the reward advantage's clip.
\end{itemize}

\paragraph{Gain matching.} The shaping coefficient $\beta$ deliberately \emph{differs} between arms,
because the quantity it multiplies differs in scale: a reachability value is a probability in
$[0,1]$, whereas a discounted cost value and a realised cost are not. Holding $\beta$ fixed would
therefore run the alternatives at a different \emph{penalty magnitude} and measure ``too much
penalty'' rather than ``worse safety signal''. We instead measured the realised magnitude of each
shaping term on a matched run and set $\beta$ so that the mean penalty matches RADPO's;
Table~\ref{tab:ablation} reports the resulting $\beta$.
% An earlier version of this ablation that shared $\beta$ across arms was discarded for exactly this reason.

\begin{table}[h]
\centering
\caption{Ablation arms. Every arm shares its task's $\eta_\lambda$, $\sigma_0$, $\alpha_0$, UTD and
interaction budget with the RADPO arm beside it, and all settings of
Table~\ref{tab:radpo-shared}.}
\label{tab:ablation}
\small
\begin{tabular}{llllr}
\toprule
Task & Arm & Cost critic & Shaping target & $\beta$ \\
\midrule
\multirow{4}{*}{CarGoal1}
 & RADPO              & reachability & $V_{reach}(s')$ & $0.05$ \\
 & discounted critic  & discounted   & $V_{c}(s')$ & $0.01$ \\
 & cost-only shaping  & - & $c_t$     & $0.05$ \\
 & lag-style shaping  & discounted   & none (cost in weights) & --- \\
\midrule
\multirow{4}{*}{PointCircle1}
 & RADPO              & reachability & $V_{reach}(s')$ & $0.05$ \\
 & discounted critic  & discounted   & $V_{c}(s')$ & $0.01$ \\
 & cost-only shaping  & - & $c_t$     & $0.05$ \\
 & lag-style shaping  & discounted   & none (cost in weights) & --- \\
\midrule
\multirow{4}{*}{HalfCheetah}
 & RADPO              & reachability & $V_{reach}(s')$ & $1.5$ \\
 & discounted critic  & discounted   & $V_{c}(s')$ & $0.254$ \\
 & cost-only shaping  & -   & $c_t$     & $2.467$ \\
 & lag-style shaping  & discounted   & none (cost in weights) & --- \\
\bottomrule
\end{tabular}
\end{table}

% ~~~~~~~~~~~~~~~~~~~~
% ~~~~~~~~~~~~~~~~~~~~
% ~~~~~~~~~~~~~~~~~~~~

\subsection{Deterministic-evaluation curves}
\label{app:hp:eval_curves}

Figure~\ref{fig:benchmark_results} in the main text plots the on-policy training rollouts, which are
the statistics the safety-multiplier controller regulates. Figure~\ref{fig:benchmark_eval} plots the
deterministic-evaluation counterpart over the same interaction budgets: $20$ evaluation episodes at
every checkpoint, with exploration noise disabled. The two views can differ appreciably in absolute
cost, thus we report both rather than choosing one.

\begin{figure*}[t]
    \centering
    \includegraphics[width=\textwidth]{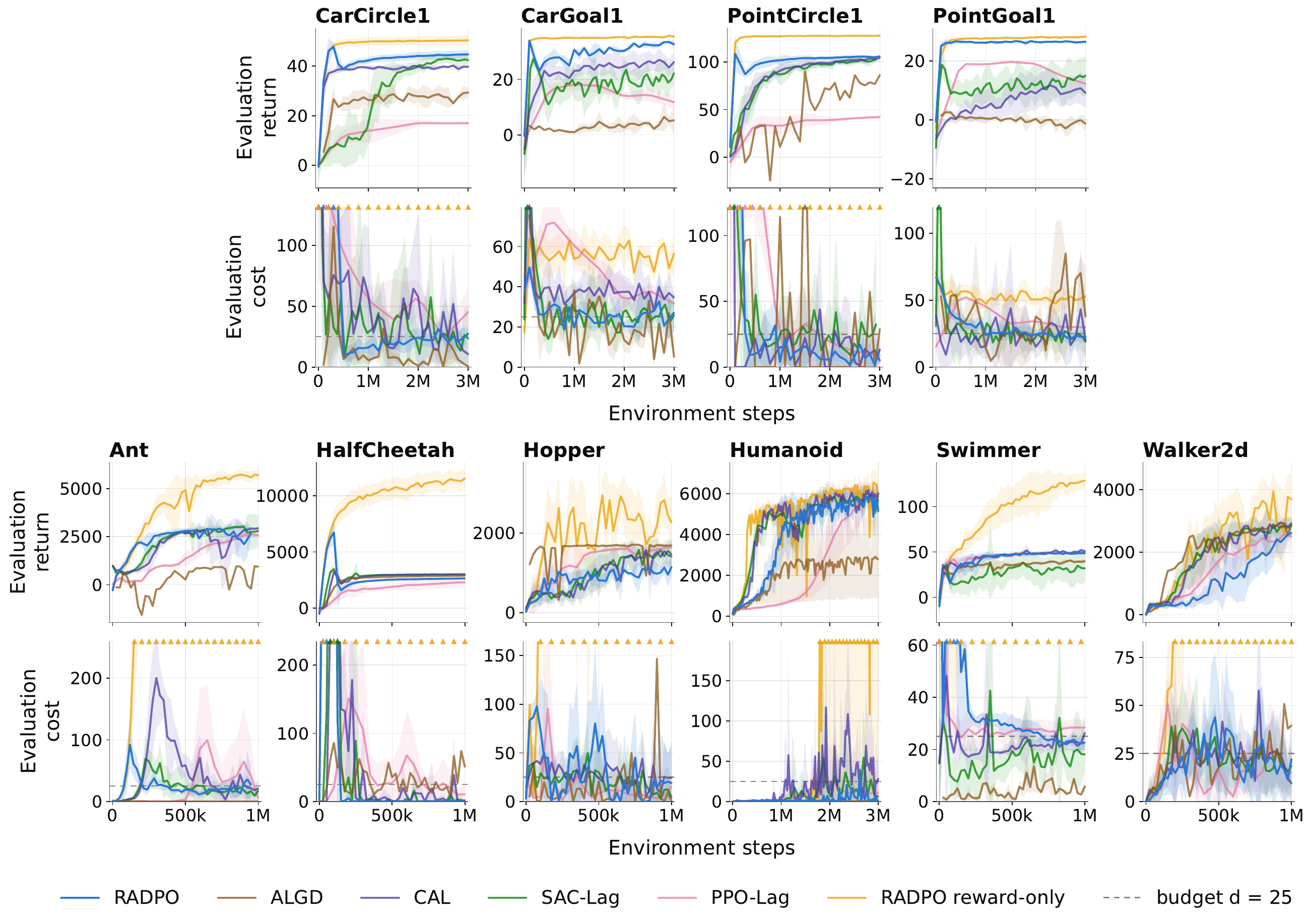}
    \caption{
    \textbf{Deterministic-evaluation return and cost across the ten tasks.}
    Evaluation return (upper row of each block) and evaluation cost (lower row), measured over $20$
    deterministic episodes per checkpoint; the training-rollout counterpart is
    Figure~\ref{fig:benchmark_results}. Curves are means over the runs of an arm with one standard
    deviation shaded, and dashed horizontal lines mark the budget $d=25$.
    }
    \label{fig:benchmark_eval}
\end{figure*}

\end{document}